\documentclass{article}
\usepackage{ijcai22}
\usepackage{booktabs}
\usepackage{amssymb}
\usepackage{spverbatim}
\usepackage{fancyvrb}
\usepackage{lipsum}
\usepackage{multirow}
\usepackage{graphicx}
\usepackage{times}
\usepackage{soul}
\usepackage{url}
\usepackage[hidelinks]{hyperref}
\usepackage[utf8]{inputenc}
\usepackage[small]{caption}
\usepackage{graphicx}
\usepackage{amsmath}
\usepackage{amsthm}
\usepackage{xspace}
\usepackage{booktabs}
\usepackage{algorithm,algpseudocode}
\usepackage[table,xcdraw]{xcolor}
\makeatletter
\usepackage{latexsym}
\usepackage{hyperref}
\usepackage{microtype}
\usepackage{booktabs}
\usepackage{url}

\makeatother

\newtheorem{theorem}{Theorem}

\title{On the Utility of Prediction Sets in Human-AI Teams}

\author{
Varun Babbar$^1$
\and
Umang Bhatt $^{1,2}$\and
Adrian Weller$^{1,2}$
\affiliations
$^1$University of Cambridge, $^2$The Alan Turing Institute\\
\emails
\{vb395, usb20, aw665\}@cam.ac.uk
}
\begin{document}

\maketitle

\begin{abstract}
Research on human-AI teams usually provides experts with a single label, which ignores the uncertainty in a model's recommendation. Conformal prediction (CP) is a well established line of research that focuses on building a theoretically grounded, calibrated prediction set, which may contain multiple labels. We explore how such prediction sets impact expert decision-making in human-AI teams.  Our evaluation on human subjects finds that set valued predictions positively impact experts. However, we notice that the predictive sets provided by CP can be very large, which leads to unhelpful AI assistants. To mitigate this, we introduce D-CP, a method to perform CP on some examples and defer to experts. We prove that D-CP can reduce the prediction set size of non-deferred examples. We show how D-CP performs in quantitative and in human subject experiments ($n=120$). Our results suggest that CP prediction sets improve human-AI team performance over showing the top-1 prediction alone, and that experts find D-CP prediction sets are more useful than CP prediction sets.
\end{abstract}

\section{Introduction}
Human-AI collaboration is of increasing importance.
Several works have shown the benefits of human-AI collaboration in boosting accuracy, fairness, and compatibility~\cite{madras2018predict,bansal2019updates,mozannar2020consistent}. One form of collaboration in the medical domain is the effect of AI explanations on team performance~\cite{Lundberg2018}, where team performance improves if model explanations are provided. Another form of human-AI collaboration develops techniques for models to defer to an expert. Prior literature exploring both these forms of collaboration has mainly considered models which output singular predictions. However, this does not allow experts to gauge and interpret the predictive uncertainty of the model, which can prevent deployment in high risk settings. A solution to this is for the model to display \textit{set valued predictions}. We define a set valued model prediction $\Gamma$ as a mapping from the input space $\mathcal{X}$ to the power set of the label space $\mathcal{Y}$, i.e. $\Gamma: \mathcal{X} \rightarrow 2^{\mathcal{Y}}$. One way to construct a set valued predictor is through a technique called Conformal Prediction (CP) \cite{vovk2005}. CP generates a prediction set that may contain multiple labels, but contains the true label with a user defined error probability. The goal of CP is to construct predictive sets that are sufficiently small but have high probability of containing the true label.

One problem often associated with CP sets is that they can be quite large, which can limit their usefulness in time and cost sensitive domains such as medical diagnostics, where it is crucial to narrow down the list of possible diagnoses. Previous work such as \cite{Bellotti} and \cite{stutz2022learning} have devised surrogate loss functions for minimizing set sizes whilst maintaining coverage guarantees. \cite{Angelopoulos2020} regularize the low scores of unlikely classes to provide small, stable sets. However, CP literature in general has given little consideration given to a) how useful such predictive sets are in human-AI teams and b) how human expertise could be leveraged to get smaller predictive sets~\cite{Sadinle2016,Romano2020,Angelopoulos_Intro_2021}. Recently, \cite{okati_cp} explored improving expert predictions using CP, with a focus on finding the optimal error tolerance parameter $\alpha$ that benefits the expert. In this concurrent work, we fix $\alpha$ and use deferral as a mechanism to provide sets that are smaller and hence more useful to an expert. 
\begin{figure}[t]
    \centering
    \includegraphics[scale=0.22]{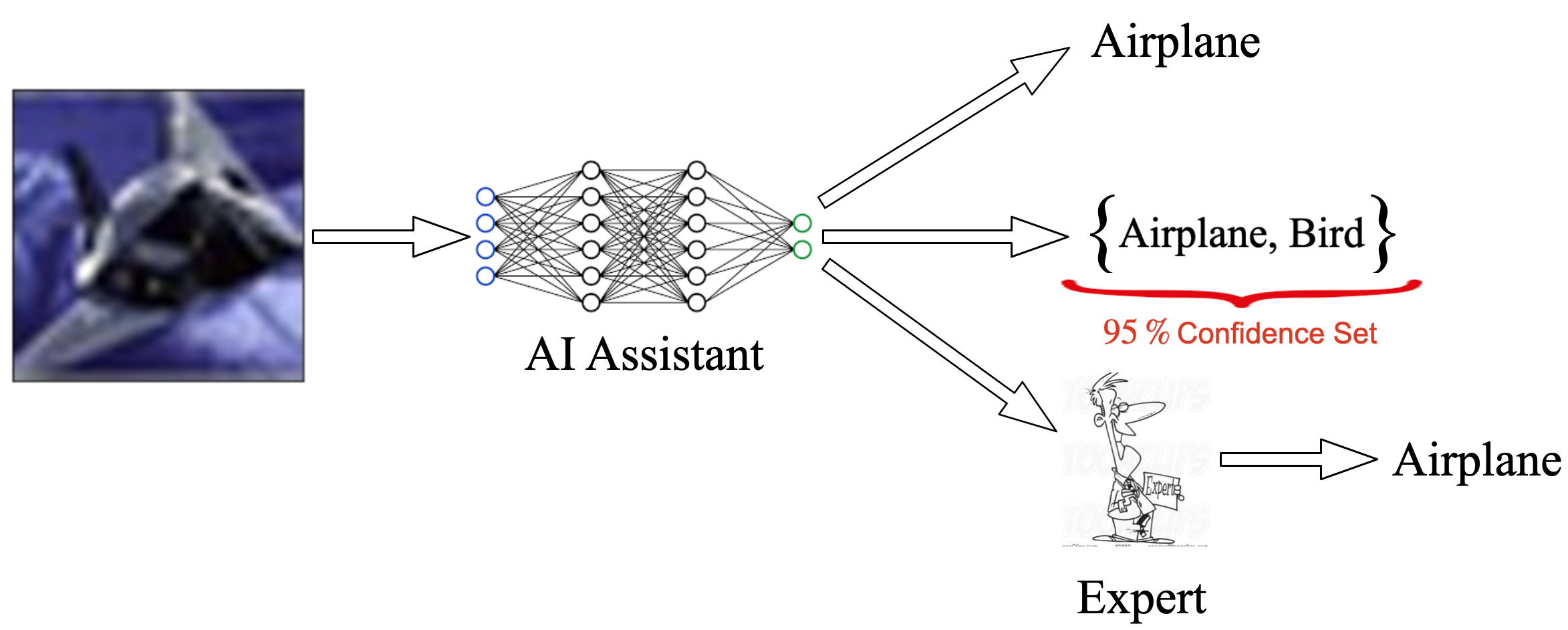}
    \caption{An AI assistant working alongside an expert can output one of three things: the most likely label, a set valued prediction with a predetermined error probability, or a deferral token indicating that the example should be labelled by the expert. The precise nature of the AI's prediction should be dependent on acquired knowledge of the expert's capabilities. Generally speaking, because the size of the predictive set is a reflection of the model's confidence, deferring examples on which an expert is more confident than the model would prevent an expert from using large, incoherent prediction sets.}
    \label{fig:example}
\end{figure}
Our contributions are the following:
\begin{itemize}
   \item Through human subject experiments on CIFAR-100, we first show that CP sets result in higher levels of reported trust and utility as compared to Top-1 predictions. 
   
   \item We empirically demonstrate one limiting aspect of CP sets: they can be very large for some examples. In order to mitigate this, we propose combining set valued classification and learning to defer. Through a toy example, we show the utility of learning to defer for the predictive set size of non-deferred examples. 

   \item We then introduce D-CP, a general practical scheme that learns to defer on some examples and perform conformal prediction on others. We prove that, under mild conditions, the set size associated with non-deferred examples will be smaller than that of the original dataset. 
   
   \item Through further human subject experiments on the CIFAR-100 dataset, we discover two benefits of D-CP: smaller predictive sets and improved team performance. We find that D-CP leads to higher reported utility and expert accuracy compared to showing only CP sets.
\end{itemize}

\section{Related Work}
\subsection{Conformal Prediction}
There is growing interest in conformal prediction \cite{vovk2005} as a method of rigorous uncertainty quantification. Given a test example $X_{test}$ and its (hidden) true label $Y_{test}$, this method allows the user to construct sets $\Gamma(X_{test})$ that control for the binary risk, i.e. the error probability $\alpha = P(Y_{test} \notin \Gamma(X_{test}))$. This is done by performing a statistical test for each label in order to decide whether the label should be present in the set. In particular, we define a \textit{conformity score} $\tau(X_{test},y): \mathcal{X} \times \mathcal{Y} \rightarrow \mathcal{R}$ that determines how different example $(X_{test},y)$ is from already observed data $\{(X_i,Y_i)\}_{i=1}^n$. This is a design choice and several papers explore different choices of conformity functions~\cite{Sadinle2016,Angelopoulos2020,Romano2020}. To include label $y$ in a predictive set, we require that the conformity score $\tau(X_{test},y)$ is at least $\alpha$-common with respect to conformity scores on previously observed data, i.e. $\textrm{Quantile}(\tau(X_{test},y),\{\tau(X_i,Y_i)\}) \geq \alpha$.

This is equivalent to learning a threshold conformity score for including labels in a set. Defining the threshold as $\tau_{cal} = \textrm{Quantile}(\alpha,\{\tau(X_i,Y_i)\})$, we require $\tau(X_{test},y) \geq \tau_{cal}$ to include the label $y$ in the predictive set for $X_{test}$. The conformal set is therefore defined as: $\Gamma(X_{test},\tau_{cal}) = \{y: \tau(X_{test},y) \geq \tau_{cal}\}$. In this paper, we employ a computationally efficient scheme called Inductive Conformal Prediction (ICP) \cite{Papadopoulos08}. This requires an additional \textbf{calibration dataset} $\mathcal{D}_{cal} =\{(X_i,Y_i)\}_{i=1}^n$ drawn from the same distribution as training and validation sets. After training a classifier on a training dataset, we can use this calibration dataset to choose the $\alpha$ Quantile threshold $\tau_{cal}$.

None of the works that have previously studied CP ~\cite{Sadinle2016,Angelopoulos2020,Romano2020,stutz2022learning} involved experts in the loop or even considered the utility of prediction sets generated in the context of human-AI teams.
\subsection{Learning to Defer}
Many works have studied the idea of learning a model that adapts to an underlying expert. One approach is to learn a rejector and a classifier, wherein, given a cost of deferring $c$, one learns a binary classifier that rejects whenever it is less than $1-c$ confident~\cite{Bartlett2008,Corinna2016}. For multi-class problems, \cite{mozannar2020consistent} learn a model that predicts the true label whenever the expert is wrong and defers otherwise. Similarly, \cite{okati2021differentiable} develop a method for exact triage under multiple expert annotations and prove its optimality under conditions where there is expert disagreement. \cite{Wilder2020}, on the other hand, develop a decision theoretic approach, training $3$ probabilistic models representing the AI, expert, and joint human-AI to maximize utility. However, all these approaches only examine settings where the AI makes point predictions whereas we aim to defer some examples and provide principled, calibrated prediction sets on others.
\section{Are Prediction Sets better for Human-AI teams than Top-1 Predictions?}
Our first study focuses on establishing the value of set valued predictions. For our experiments, we focus on one particular CP scheme called Regularised Adaptive Prediction Sets (RAPS) \cite{Angelopoulos2020}. We recruit $30$ participants on Prolific, paying them at a rate of £$10$ per hour prorated, and divide them into $2$ equal groups. The first group is shown $18$ images from the CIFAR-100 dataset alongside the model's most probable prediction (Top-1). The second group is shown the same images but alongside a RAPS prediction set with error rate $\alpha = 0.1$. To understand the effect of set valued predictions on examples of varying difficulty, we divide the CIFAR-100 test dataset into $3$ difficulty quantiles, where difficulty is defined as the entropy of the model predictive distribution. We select $6$ images from each difficulty quantile. For each quantile, we show $2$ images whose Top-1 prediction is incorrect but whose RAPS set contains the true label. This is consistent with the accuracy of the model ($\approx 65\%$) and lets us determine the effect of set valued predictions on examples on which the model is \textit{almost} correct.
Given these model predictions for each image, we ask participants in both groups to predict the correct class, rate their confidence in their predictions, and rate how useful they found the model predictions on that example.
At the end of the survey, we ask participants to rate their overall trust in the model's predictions. All ratings are on a scale from $1-10$. We employ preliminary attention checks by first asking them to classify $3$ easy examples, rejecting any participants who classify these examples incorrectly. We evaluate the statistical significance of our results using a two sample t-test.

\begin{table}[htb]
\centering
\resizebox{8.5cm}{!}{%
\begin{tabular}{ccccc}
\hline
\textbf{Metric}                                                   & \textbf{Top-1}    & \textbf{RAPS}     &  \textbf{$p$ value}   &    \textbf{Effect Size} \\ \hline
\textbf{Accuracy} & 0.76 $\pm$ 0.05 & 0.76 $\pm$ 0.05 & 0.999 & 0.000             \\ 
\textbf{Reported Utility}                                                  & 5.43 $\pm$ 0.69 & 6.94 $\pm$ 0.69 & \textbf{0.003} & 1.160         \\ 
\textbf{Reported Confidence}                                               & 7.21 $\pm$ 0.55 & 7.88 $\pm$ 0.29 & 0.082 & 0.674        \\ 
\textbf{\begin{tabular}[c]{@{}c@{}}Reported Trust in Model \end{tabular}}                                                    & 5.87 $\pm$ 0.81 & 8.00 $\pm$ 0.69 & \textbf{$<$ 0.001} & 1.487           \\ \hline
\end{tabular}
}
\caption{Top-1 vs RAPS: All Examples}
\label{tab:top_1_vs_raps_all}
\end{table}
\begin{table}[htb]
\centering
\resizebox{8.25cm}{!}{%
\begin{tabular}{ccccc}
\hline
\textbf{Metric}                                                   & \textbf{Top-1}    & \textbf{RAPS}     & \textbf{$p$ value} & \textbf{Effect Size} \\\hline
\textbf{Accuracy} & 0.90 $\pm$ 0.05 & 0.87 $\pm$ 0.07 & 0.486 & 0.273             \\ 
\textbf{Reported Utility}                                                  & 6.067 $\pm$ 0.94 & 6.35 $\pm$ 1.00 & 0.438 & 0.195            \\ 
\textbf{Reported Confidence}                                               & 7.88 $\pm$ 0.65 & 8.82 $\pm$ 0.31 & \textbf{0.013} & 1.019            \\ \hline
\end{tabular}
}
\caption{Top-1 vs RAPS: Lowest Difficulty Quantile}
\label{tab:top_1_vs_raps_lowest}
\end{table}
\begin{table}[htb]
\centering
\resizebox{8.5cm}{!}{%
\begin{tabular}{ccccc}
\hline
\textbf{Metric}                                                   & \textbf{Top-1}    & \textbf{RAPS}     & \textbf{$p$ value} & \textbf{Effect Size}\\ \hline
\textbf{Accuracy} & 0.64 $\pm$ 0.07 & 0.66 $\pm$ 0.10 & 0.828 & 0.068            \\
\textbf{Reported Utility}                                                  & 5.30 $\pm$ 0.75 & 7.28 $\pm$ 0.69 & \textbf{0.001} & 1.432            \\ 
\textbf{Reported Confidence}                                               & 6.64 $\pm$ 0.64 & 6.96 $\pm$ 0.78 & 0.4888 & 0.280            \\ \hline
\end{tabular}
}
\caption{Top-1 vs RAPS: Highest Difficulty Quantile}
\label{tab:top_1_vs_raps_highest}
\end{table}
From Table \ref{tab:top_1_vs_raps_all}, we see that Top-1 predictions result in statistically significant lower levels of trust ($p < 0.001$) and perceived utility ($p = 0.003$) compared to RAPS. However, both schemes result in similar accuracy and confidence in predictions. We also see that users find Top-1 and RAPS predictions equally useful for easy examples (Table \ref{tab:top_1_vs_raps_lowest}). This makes sense because in such cases, the predictive set will be small and therefore comparable to a Top-1 prediction. However, users are more confident about their answers when they observe RAPS predictions. On the other hand, RAPS sets are perceived to be much more useful on hard examples, where Top-1 predictions will often be wrong. \\
\noindent
\textbf{Takeaway:} \textit{While there is seen to be no significant difference in team accuracy when shown either Top-1 or set-valued predictions, displaying set-valued predictions in human-AI teams results in higher reported utility of predictions as well as higher reported overall trust in the model.}

\section{Proposed Approach}
\subsection{The Problem with CP Sets}
In our experiments above, we showed users examples where the set sizes on CIFAR-100 are small enough to be considered useful. However, this may not always be the case, especially on tasks with large label spaces. For instance, a standard WideResNet model trained on CIFAR-100 ($\approx 65\%$ accuracy) with APS conformal prediction yields prediction sets with greater than $15$ labels for $20\%$ of examples. One option to mitigate this issue is to defer examples with the largest CP set sizes to an expert. However, this provides no guarantee that the expert will be able to classify them with sufficient accuracy. Furthermore, we also lose the finite sample coverage guarantees provided by contemporary CP methods, i.e. we cannot ascertain that $P(Y_{test} \notin \Gamma(X_{test})) \leq \alpha$.

\subsection{Our Scheme}
Our scheme, described in Algorithm~\ref{alg:algorithm}, is centered around two components: a deferral policy $\pi(x): \mathcal{X} \rightarrow \{0,1\}$ and a CP method. The deferral policy is based on our knowledge of the expert's strengths either acquired during training or a-priori. For example, if an expert is better at identifying brain tumors than our model, our policy should learn to defer those examples with high probability. Using this black box policy, we first prune our calibration dataset, removing all examples where our deferral policy recommends deferral. One could use any scheme in ~\cite{mozannar2020consistent,okati2021differentiable,Wilder2020} to learn a deferral policy. While Algorithm~\ref{alg:algorithm} specifies a deferral policy as an input, for some deferral methods (such as \cite{mozannar2020consistent}), the policy is trained alongside the model. In others, such as \cite{okati2021differentiable}, the policy is applied post-hoc. In this paper, we consider the former deferral policy: the D-CP algorithm for this is outlined in Algorithm $2$ in the Supplementary Material.
After training a model and a suitable deferral policy, we perform conformal calibration on this pruned dataset of non-deferred examples. In this procedure, for any predictive set $\Gamma(X_{test},\tau_{cal})$ for an example $X_{test}$ we can guarantee that:
\begin{equation}
    \label{eqn:coverage_guarantee}
    1-\alpha \leq P(Y \in \Gamma(X_{test},\tau_{cal})|\pi(X_{test}) = 0)
\end{equation}
where $1$ represents the action of deferral. From \cite{Angelopoulos2020}, when the conformity scores are known to be almost surely distinct and continuous, we can also guarantee:
\begin{equation}
    \label{eqn:coverage_guarantee_distinct}
    \text{\footnotesize{$P(Y \in \Gamma(X_{test},\tau_{cal})|\pi(X_{test}) = 0) \leq 1-\alpha + \frac{1}{n+1}$}}
\end{equation}
where $n$ is the size of the non-deferred calibration dataset. Because the deferral policy $\pi$ probabilistically decides which unseen examples to defer, all non-deferred examples can be thought of as being generated from a data generating distribution $X \sim p(X|\pi(X) = 0)$. Any new test example $X_{test}$ that is not deferred is therefore independently drawn from this distribution. Thus, $\{X_i\}_{i=1}^{n}\cup \{X_{test}\} \sim p(X_1,..X_{test}|\pi(X_1),..\pi(X_{test})  = 0)$ are exchangeable, thereby satisfying the coverage guarantee in Equation \ref{eqn:coverage_guarantee}.
\begin{figure}[htb]
    \centering
    \includegraphics[scale=0.36]{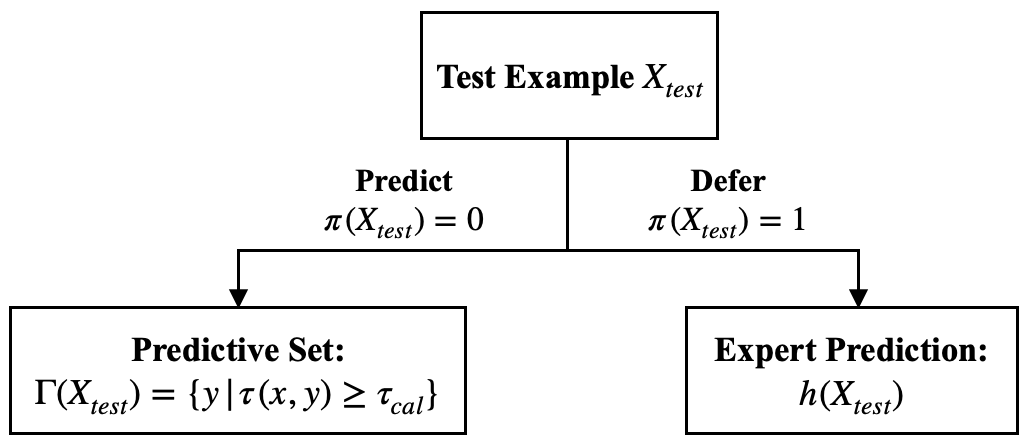}
    \caption{D-CP: Test Phase given a deferral policy $\pi(X)$}
    \label{fig:D-CP_Pipeline_2}
    \vspace{-0.2cm}
\end{figure}
\begin{figure}[htb]
    \centering
    \includegraphics[scale=0.36]{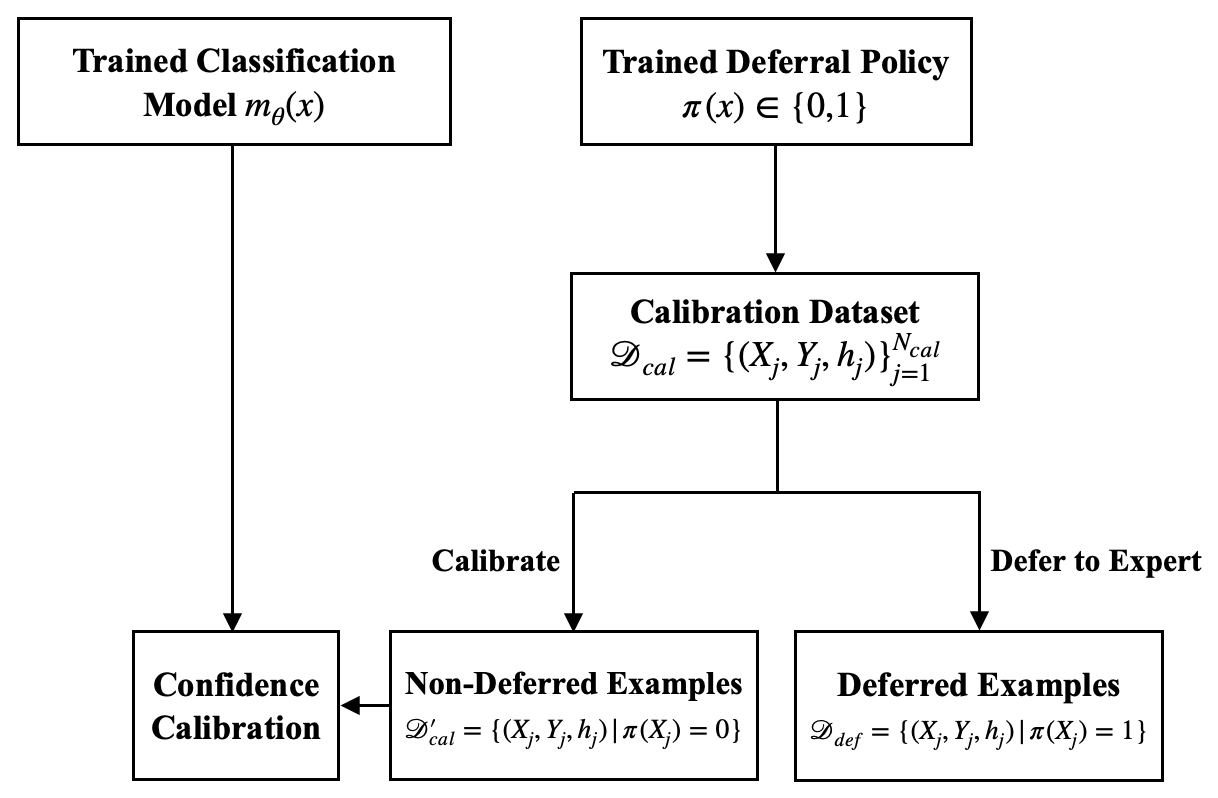}
    \caption{D-CP: Training and Calibration Phase}
    \label{fig:D-CP_Pipeline}
\end{figure}

\begin{algorithm}[htb]
\caption{General D-CP}
\label{alg:algorithm}
\textbf{Input}: Classifier $m_\theta(x) \in \mathcal{R}^{|\mathcal{Y}|}$, Deferral Policy $\pi(x) \in \{0,1\}$, Training Set $\mathcal{D}$, Expert $h(x) \in \mathcal{Y}$, Calibration Set $\mathcal{D}_{\textrm{cal}}$, Validation Set $\mathcal{D}_{\textrm{val}}$, Test Example $x_{test}$, Conformity Score Function $\tau(X,y)$, Loss function $l(m_\theta(x),y,h(x))$\\
\textbf{Parameter}:  Number of Epochs $N$, Learning Rate $\gamma$, Error Tolerance $\alpha$
\begin{algorithmic}[1] 
\For {$i \in \{1,...N\}$}
\For {Batch $\mathcal{B} \in \mathcal{D}$}
    \State $\theta = \theta - \gamma\mathbb{E}_{(x,y) \in \mathcal{B}}[\nabla_{\theta} l(m_\theta(x),y,h(x))]$
\EndFor
\EndFor
\State $\mathcal{D}^\prime_{\textrm{cal}} = \{(X,Y)| \pi(X) = 0, (X,Y) \in \mathcal{D}_{\textrm{cal}}\}$
\State $\tau_{cal} = \textrm{Quantile}(\alpha, \{\tau(X_i,Y_i)|(X_i,Y_i) \in \mathcal{D}^\prime_{\textrm{cal}}\})$ 
\If {$\pi(X_{test}) = 0$}
\State \textbf{return} $\Gamma(X_{test},\tau_{cal}) = \{y|\tau(X_{test},y) \geq \tau_{cal}\}$
\ElsIf {$\pi(X_{test}) = 1$}
\State \textbf{Defer to Expert} $h(X_{test})$
\EndIf
\end{algorithmic}
\end{algorithm}
To show the utility of our scheme, a good deferral policy would guarantee that resulting predictive sets on non-deferred examples will contain fewer incorrect labels than before. We prove this formally in Theorem $1$ in the Supplementary Material.
While this technically applies to conformity scores $\tau(x,y)$ that are monotonic with respect to softmax probabilities (such as LAC), our subsequent experiments with other CP methods such as RAPS and APS suggest that our scheme generalises well across other classes of conformity score functions.

\section{Toy Example}
One way to combine conformal prediction and deferral is to only perform CP on ``easy'' examples and defer the ``hard'' examples. An ``easy'' example would be one which the model is confident on and a ``hard'' example is the converse. This can lead to smaller sets. To demonstrate the intuition, we generate equiprobable synthetic data using a Mixture of Gaussians (MoG) model. Each datapoint is generated from one of $4$ Gaussians - $\mathcal{N}(1,1)$, $\mathcal{N}(1,-1)$, $\mathcal{N}(-1,1)$, and $\mathcal{N}(-1,-1)$ - and we wish to infer class memberships. 
We first train a multilayer perceptron (MLP) on $1000$ training samples (not shown) to infer the decision boundaries. Then, using a held out calibration set, we perform CP with error tolerance $\alpha = 5\%$ using the Least Ambiguous Classifiers (LAC) method~\cite{Sadinle2016}. In this method, we use the model softmax probabilities $p(y|x) = \tau(x,y)$ as conformity scores. Figure \ref{fig:bar_plots} (top) shows a 1-D scatter plot of conformity scores assigned to ground truth labels in the toy dataset. 

Figure \ref{fig:toy_example_part_1} shows the resulting test datapoints colored according to their true classes with model decision boundaries overlaid. We see that points closer to the decision boundary have larger predictive set sizes, reflecting their inherent uncertainty. If we defer points with conformity scores in the bottom $15^{th}$ percentile (naive decision policy) as in Figure \ref{fig:bar_plots}, the $\alpha$ threshold conformity score will increase. From Figure \ref{fig:bar_plots} (Right), for non-deferred examples, this increases the threshold for including labels in the set, resulting in more confident sets for the same error control. However, this naive deferral method, whilst ensuring small set sizes on the remaining examples, does not take into account the expertise of the expert involved. Furthermore, we assumed access to ground truth labels for test examples, which is not practical. We can engage the expert in a better manner and approximate the idea of the toy example by learning a \textit{deferral policy} which incorporates estimates of expert ability as well as machine difficulty.  This scheme makes an implicit assumption that the expert is a) either better than the model on average or b) not necessarily better than the model on average, but is proficient in classifying certain subgroups of examples. In these situations, our deferral policy is more likely to defer examples that a model is less confident on. Given these assumptions about an expert, Theorem $1$ ensures lower predictive set sizes. 

\begin{figure}[htb]
    \centering
    \includegraphics[scale=0.25]{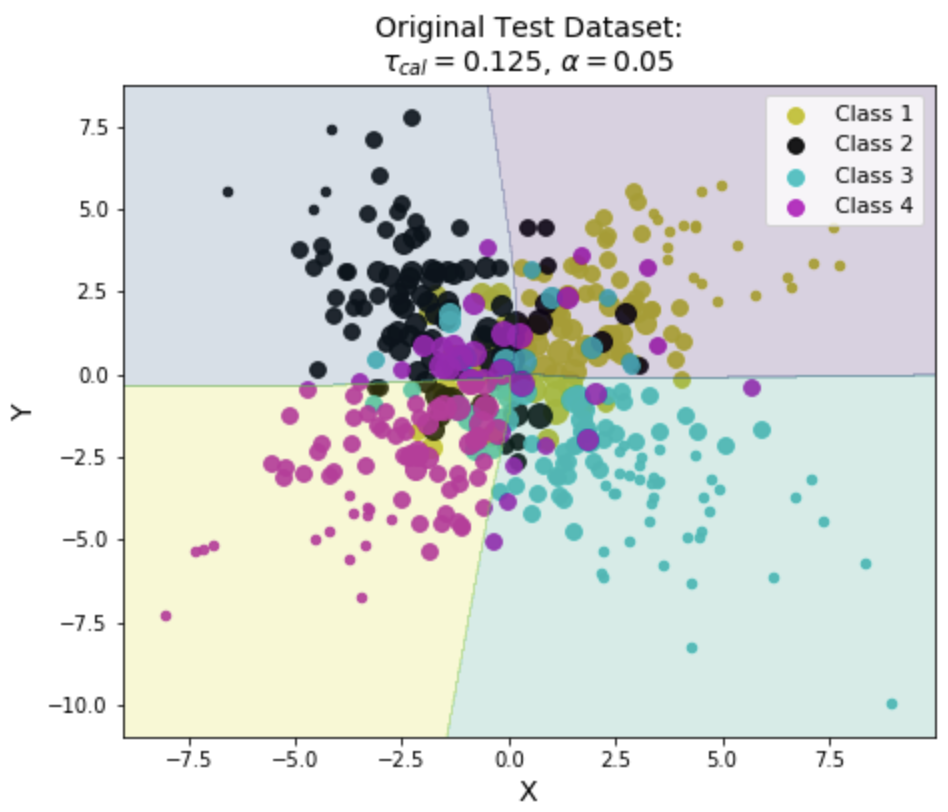}
    \includegraphics[scale=0.25]{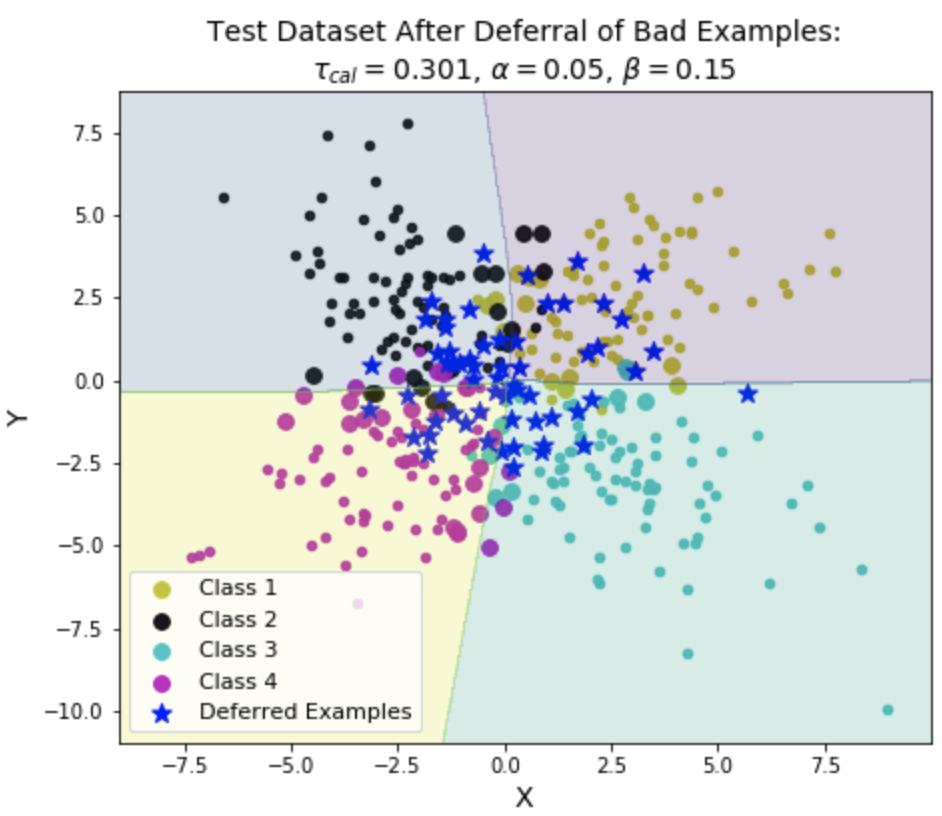}
    \caption{(Left): Toy dataset from Figure 1 comprising of datapoints belonging to one of $4$ classes along with overlaid model decision boundaries. \textbf{The size of the datapoints indicates their predictive set sizes}. (Right) We defer the $\beta = 0.15$ proportion of examples with the lowest ground truth conformity scores. Doing so increases the value of the $5^{th}$ percentile conformity score of the remaining examples in Figure 1, causing CP set sizes of examples to be smaller. Note that we have not changed the model in this process.}
    \label{fig:toy_example_part_1}
\end{figure}
\begin{figure}[htb]
    \centering
    \includegraphics[scale=0.31]{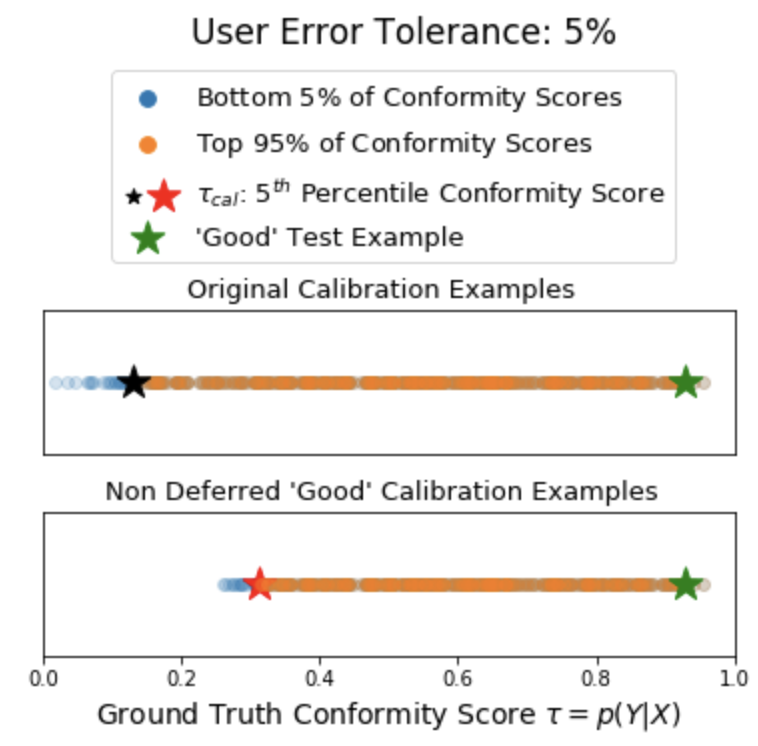}\vspace{0.5cm}
    \includegraphics[scale=0.31]{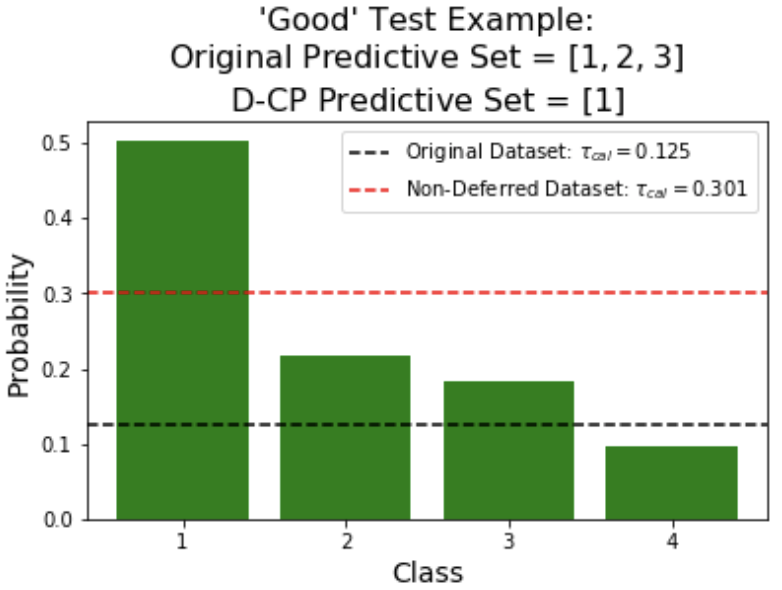}
        \caption{(Left): 1-D scatter plot of all ground truth conformity scores $\tau = p(Y_i|X_i)$ on a toy calibration dataset in Figure 2. We assume an oracle deferral policy that defers $\beta = 15\%$ of examples with the lowest $\tau$. Both values of $\tau_{cal}$ maintain $95\%$ coverage on their respective datasets. (Right): Class probabilities for the test \textbf{green starred example}. For the predictive set, we include all scores which are greater than the threshold $\tau_{cal}$. Thus, the predictive set $\{1,2,3\}$ gives $95\%$ coverage for the original dataset. On the non-deferred dataset, the set $\{1\}$ gives $95\%$ coverage.}
    \label{fig:bar_plots}
\end{figure}\vspace{-3mm}


\section{Experiments with D-CP}
To validate our approach, we perform experiments with synthetic expert labels on the CIFAR-100 dataset and real expert labels on the CIFAR-10H \cite{cifar10h} dataset\footnote{Our code is hosted at \url{https://github.com/cambridge-mlg/d-cp}.}. Because the CIFAR-10H dataset contains expert labels only on the CIFAR-10H validation set, we employ the approach in \cite{mozannar2020consistent} and train a binary classifier to predict examples where the expert is correct. We then provide synthetic expert labels $\mathcal{I}_{h(x)= y}$ or $\mathcal{I}_{h(x) \neq y}$ for examples the training set according to whether the expert errs on them. Note that, in line with the assumption made in this paper, the experts chosen in this setting are, on average, better than the model trained. We consider $2$ scenarios:
\begin{itemize}
    \item We have access to a single expert's annotations. For CIFAR-100, we generate a synthetic expert with $70\%$ accuracy. To motivate this choice, we ran a control study where we asked $20$ participants to classify $15$ randomly chosen CIFAR-100 examples. We found participants had average accuracy of $69\%$ with a standard error of $\approx 2.5\%$. For the CIFAR-10H dataset, we randomly sample a label from the predictive distribution provided.
    \item We have access to multiple expert annotations. This is an ensemble of the above experts, and the predicted class is chosen through majority voting for both datasets. For the CIFAR-100, we generate predictions from $5$ experts. 
\end{itemize}

\begin{table*}[htb]
\resizebox{\textwidth}{!}{%
\begin{tabular}{@{}ccccccc@{}}
\toprule
\rowcolor[HTML]{FFFFFF} 
\cellcolor[HTML]{FFFFFF} &
  \cellcolor[HTML]{FFFFFF} &
  \cellcolor[HTML]{FFFFFF} &
  \cellcolor[HTML]{FFFFFF} &
  \multicolumn{3}{c}{\cellcolor[HTML]{FFFFFF}\scriptsize{\textbf{Predictive Set Size of Non-Deferred Examples}}} \\ \cmidrule(l){5-7} 
\rowcolor[HTML]{FFFFFF} 
\scriptsize{\multirow{-2}{*}{\cellcolor[HTML]{FFFFFF}\textbf{\begin{tabular}[c]{@{}c@{}}Deferral \\ Rate $b$\end{tabular}}}} &
  \scriptsize{\multirow{-2}{*}{\cellcolor[HTML]{FFFFFF}\textbf{\begin{tabular}[c]{@{}c@{}}Classifier\\ Accuracy\end{tabular}}}}&
  \scriptsize{\multirow{-2}{*}{\cellcolor[HTML]{FFFFFF}\textbf{\begin{tabular}[c]{@{}c@{}}System Accuracy\\ (Single Expert)\end{tabular}}}} &
  \scriptsize{\multirow{-2}{*}{\cellcolor[HTML]{FFFFFF}\textbf{\begin{tabular}[c]{@{}c@{}}System Accuracy\\ (Multiple Experts)\end{tabular}}}} &
  \scriptsize{\textbf{RAPS}} &
  \scriptsize{\textbf{APS}} &
  \textbf{\scriptsize{LAC}} \vspace{-1mm}\\ \midrule
\vspace{-1mm}
\scriptsize{0} &
  \scriptsize{65.18 $\pm$ 0.30} &
  \scriptsize{65.18 $\pm$ 0.30} &
  \scriptsize{65.18 $\pm$ 0.30} &
  \scriptsize{3.75 $\pm$ 0.06} &
  \scriptsize{4.61 $\pm$ 0.08} &
  \scriptsize{3.26 $\pm$ 0.03} \\
  \vspace{-1mm}
\scriptsize{0.05} &
  \scriptsize{68.39 $\pm$ 0.31} &
  \scriptsize{68.04 $\pm$ 0.32} &
  \scriptsize{68.91 $\pm$ 0.33} &
  \scriptsize{3.22 $\pm$ 0.05} &
  \scriptsize{4.16 $\pm$ 0.06} &
  \scriptsize{2.48 $\pm$ 0.03} \\ 
  \vspace{-1mm}
\scriptsize{0.10} &
  \scriptsize{69.92 $\pm$ 0.24} &
  \scriptsize{69.95 $\pm$ 0.31} &
  \scriptsize{71.53 $\pm$ 0.35} &
 \scriptsize{2.81 $\pm$ 0.05} &
  \scriptsize{4.05 $\pm$ 0.06} &
  \scriptsize{2.13 $\pm$ 0.04} \\ 
  \vspace{-1mm}
\scriptsize{0.20} &
  \scriptsize{72.98 $\pm$ 0.30} &
  \scriptsize{72.25 $\pm$ 0.30} &
  \scriptsize{78.99 $\pm$ 0.40} &
  \scriptsize{2.36 $\pm$ 0.07} &
  \scriptsize{2.93 $\pm$ 0.10} &
  \scriptsize{2.07 $\pm$ 0.03}\vspace{1mm} \\  \bottomrule
\end{tabular}
}
\resizebox{\textwidth}{!}{%
\begin{tabular}{@{}ccclccc@{}}
\toprule
\rowcolor[HTML]{FFFFFF} 
\cellcolor[HTML]{FFFFFF} &
  \cellcolor[HTML]{FFFFFF} &
  \cellcolor[HTML]{FFFFFF} &
  \multicolumn{1}{c}{\cellcolor[HTML]{FFFFFF}} &
   \multicolumn{3}{c}{\cellcolor[HTML]{FFFFFF}\scriptsize{\textbf{Predictive Set Size of Non-Deferred Examples}}} \\ \cmidrule(l){5-7} 
\rowcolor[HTML]{FFFFFF}
 \scriptsize{\multirow{-2}{*}{\cellcolor[HTML]{FFFFFF}\textbf{\begin{tabular}[c]{@{}c@{}}Deferral \\ Rate $b$\end{tabular}}}} &
   \scriptsize{\multirow{-2}{*}{\cellcolor[HTML]{FFFFFF}\textbf{\begin{tabular}[c]{@{}c@{}}Classifier\\ Accuracy\end{tabular}}}} &
   \scriptsize{\multirow{-2}{*}{\cellcolor[HTML]{FFFFFF}\textbf{\begin{tabular}[c]{@{}c@{}}System Accuracy\\ (Single Expert)\end{tabular}}}} &
  \scriptsize{\multirow{-2}{*}{\cellcolor[HTML]{FFFFFF}\textbf{\begin{tabular}[c]{@{}c@{}}System Accuracy\\ (Multiple Experts)\end{tabular}}}} &
   \scriptsize{\textbf{RAPS}} &
   \scriptsize{\textbf{APS}} &
  \scriptsize{\textbf{LAC}} \vspace{-1mm}\\ \midrule
  \vspace{-1mm}
\scriptsize{0} &
 \scriptsize{ 82.02 $\pm$ 0.55} &
  \scriptsize{82.02 $\pm$ 0.55} &
  \scriptsize{82.02 $\pm$ 0.55} &
  \scriptsize{1.91 $\pm$ 0.03} &
  \scriptsize{2.83 $\pm$ 0.05} &
  \scriptsize{2.47 $\pm$ 0.12} \\
  \vspace{-1mm}
\scriptsize{0.05} &
  \scriptsize{84.41 $\pm$ 0.69} &
  \scriptsize{84.31 $\pm$ 0.65} &
  \scriptsize{84.64 $\pm$ 0.30} &
  \scriptsize{1.87 $\pm$ 0.08} &
  \scriptsize{2.76 $\pm$ 0.10} &
  \scriptsize{2.25 $\pm$ 0.08} \\
  \vspace{-1mm}
\scriptsize{0.10} &
  \scriptsize{86.12 $\pm$ 0.67} &
  \scriptsize{86.53 $\pm$ 0.68} &
  \scriptsize{88.12 $\pm$ 0.61} &
  \scriptsize{1.73 $\pm$ 0.08} &
  \scriptsize{2.56 $\pm$ 0.07} &
  \scriptsize{1.90 $\pm$ 0.15} \\ 
  \vspace{-1mm}
\scriptsize{0.20} &
  \scriptsize{88.97 $\pm$ 0.50} &
  \scriptsize{89.43 $\pm$ 0.64} &
  \scriptsize{91.46 $\pm$ 0.32} &
  \scriptsize{1.49 $\pm$ 0.06} &
  \scriptsize{2.13 $\pm$ 0.11} &
  \scriptsize{1.51 $\pm$ 0.09} \\  \bottomrule
\end{tabular}%
}
\caption{D-CP Predictive Set Size, System Accuracy, and Classifier Accuracy on the CIFAR-100 (top) and CIFAR-10H (bottom) datasets for the deferral scheme in \protect\cite{mozannar2020consistent} and the 3 CP schemes ($\alpha = 0.1$, 5 Trials, $95\%$ CI). Even in the low deferral rate regime, we not only obtain smaller set sizes across all CP schemes tested, but also benefit from increased human-AI system accuracy. While having multiple experts does not further improve the predictive set size for this deferral policy, we benefit from further improved system accuracy.}
\label{tab:cifar10h_results}
\end{table*}

Our deferral policy is based on the loss function in \cite{mozannar2020consistent}. We train a WideResNet \cite{zagoruyko2017wide} classifier $m_\theta(x): \mathcal{X} \rightarrow \mathcal{Y} \cup \perp$ on CIFAR-10H and CIFAR-100 for $5$ and $10$ epochs respectively using the learning rate schedule in \cite{mozannar2020consistent}. $\perp$ represents the action of deferral to an expert $h(x)$. We modify the loss function used in this work as below:
    \begin{align*}
        \label{eqn:consistent_estimators_loss}
        \mathcal{L}_{\textrm{CE}}(h,x,y,m_\theta) &= -(\mathcal{I}_{h(x)\neq y} + \alpha \mathcal{I}_{h(x) = y})\log m_\theta(x)_y\\&-\beta\mathcal{I}_{h(x)=y}\log m_\theta(x)_{\perp}
    \end{align*}
    where we set $\alpha = 1$ and vary the $\beta \in [0,1]$ penalty term to control the deferral rate. The policy $\pi(x)$ is therefore:
        \[ \pi(x) = \begin{cases} 
            \label{eqn:policy}
          1 & \textrm{argmax}_{y \in \mathcal{Y}\cup\perp} m_\theta(x)_y = |\mathcal{Y}\cup\perp| \\
          0 & \textrm{otherwise}
        \end{cases}
        \]
    To compute conformity scores, we renormalize the softmax probabilities for examples where $\pi(x) = 0$ using Bayes' rule:
    \begin{align*}
     p(y|x,\pi(x) = 0,\theta) &= \frac{p(y\neq |\mathcal{Y}\cup\perp||x,y,\theta)p(y|x,\theta)}{p(y\neq |\mathcal{Y}\cup\perp||x,\theta)} \\
                       &= \frac{p(y|x,\theta)}{p(y \neq |\mathcal{Y}\cup\perp||x,\theta)}
\end{align*}
\vspace{-3mm}
\begin{figure*}[tb]
   \centering
       \includegraphics[scale=0.4]{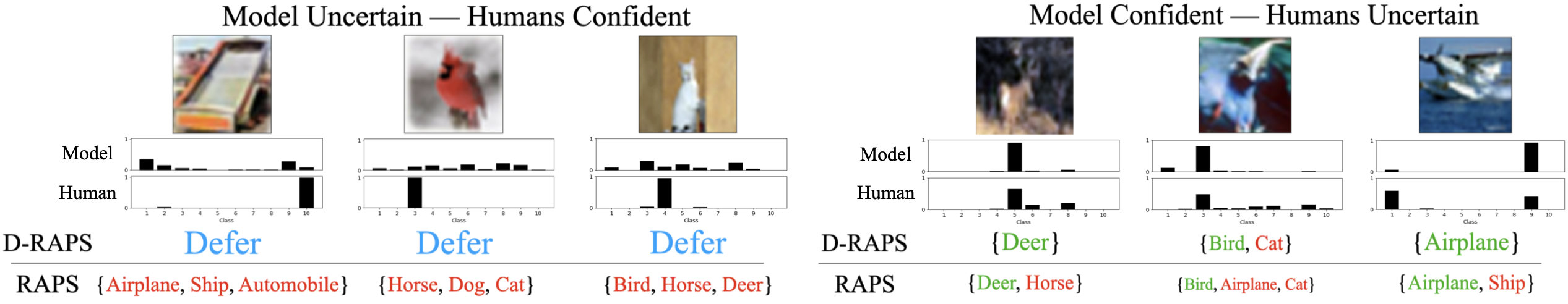}
    \caption{D-RAPS vs RAPS on some examples in the CIFAR-10H dataset ($\alpha = 0.05$, $b = 0.2$). Deferring on examples where experts are more confident than the model provides smaller sets on examples where the model is more confident than the expert. Thus, we leverage the strengths of both the model and the expert.} \label{fig:dcp_examples}
\end{figure*}
\begin{figure}[H]
    \centering
    \includegraphics[scale=0.27]{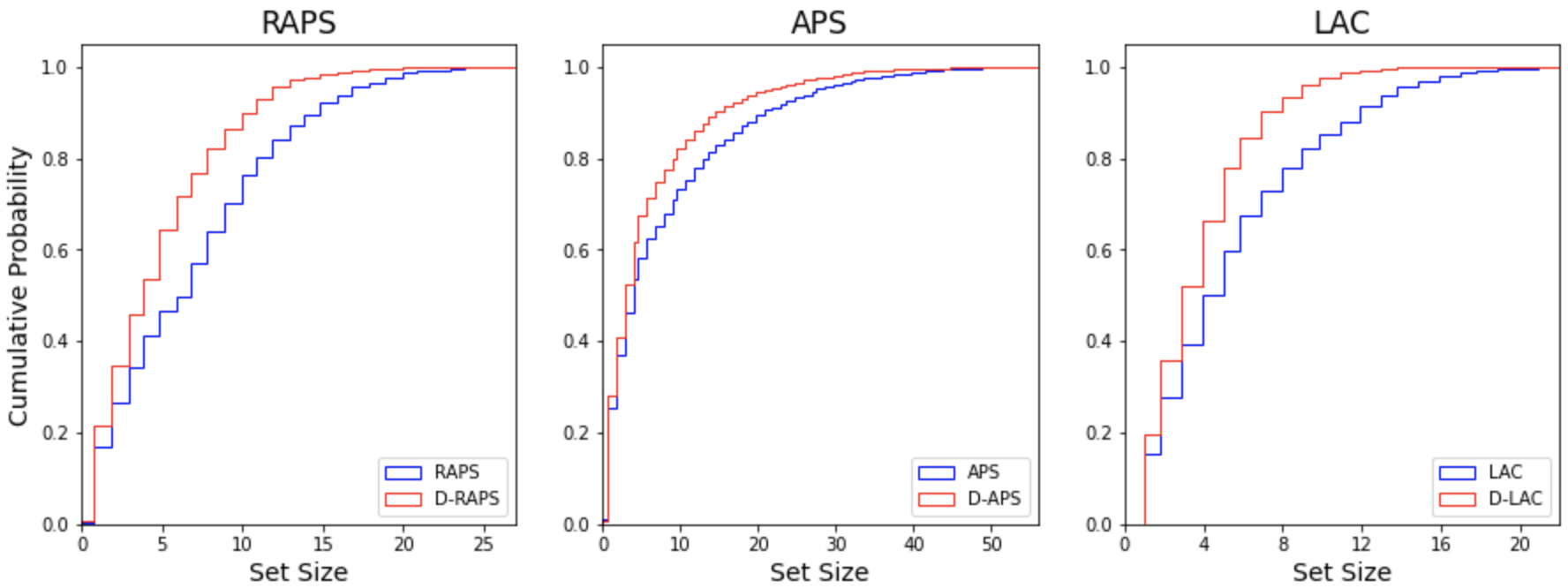}
    \caption{Cumulative CP and D-CP Set Size Distribution of Non-Deferred Examples in the CIFAR-100 dataset ($\alpha = 0.05$, deferral rate $b = 0.2$, Single Expert).}
    \label{fig:cumulative_set_size_dist_dcp}
\end{figure}
\vspace{-2mm}
In our experiments, we did not notice any statistically significant difference in accuracy of non-deferred examples or predictive set sizes when employing multiple experts as opposed to a singular expert, at least in the deferral rate regimes tested. Because we are performing experiments in the low deferral rate regime, it is likely that the deferral scheme defers similar examples to both expert types - examples the model is sure the expert(s) will get right. Thus, in Table \ref{tab:cifar10h_results}, the classifier accuracy and predictive set sizes are representative for both singular and multiple experts. However, we benefit from increased system accuracy by using ensemble voting across multiple experts. In addition, per Table \ref{tab:cifar10h_results} and Figure \ref{fig:cumulative_set_size_dist_dcp}, our scheme ensures smaller set sizes across all conformal methods and deferral rates tested. Increasing the deferral rate reduces the predictive set size. In Figure  \ref{fig:dcp_examples}, the model and expert have a mutually beneficial relationship: the model provides smaller predictive sets on examples the expert is more uncertain on and defers examples it is less certain of than an expert.

\textbf{Takeaway:} \textit{D-CP provides smaller predictive set sizes on non-deferred examples for the same level of coverage. For the policy in \cite{mozannar2020consistent}, while the number of experts does not make a difference in the resulting predictive set size of non-deferred examples, having more experts predict through majority voting improves the system accuracy.}


\section{Evaluation on Experts}
Our second human subject experiment focuses on establishing the value of smaller set predictions and learning to defer - the 2 promises of D-CP. We choose another set of $15$ examples from the CIFAR-100 test set for which we generate RAPS prediction sets with error rate $\alpha = 0.1$ and D-RAPS prediction sets with deferral rate $0.2$ and error rate $\alpha = 0.1$. We select $12$ non-deferred examples at random wherein the D-RAPS predictive set is smaller than the RAPS predictive set, but the ground truth labels are contained in both sets. Lastly, we choose the remaining $3$ deferred examples where the model is underconfident, i.e. the ground truth label is not in the RAPS set. This aims to establish the value of deferral in situations where the model may provide misleading predictions. We ask participants the same questions as in Section $3$ and follow a similar recruitment procedure as in Section 3 ($60$ participants total, $2$ groups, reward of £$10$ per hour prorated).

\begin{table}[H]
\centering
\resizebox{8.5cm}{!}{%
\begin{tabular}{ccccc}
\hline
\textbf{Metric}            & \textbf{D-RAPS}   & \textbf{RAPS}     & \textbf{$p$ value} & \textbf{Effect Size} \\ \hline
\textbf{Accuracy}          & 0.76 $\pm$ 0.08 & 0.67 $\pm$ 0.05 & \textbf{0.002} & 0.832       \\ 
\textbf{Reported Utility} & 7.93 $\pm$ 0.39 & 6.32 $\pm$ 0.60 & $<$ \textbf{0.001} & 1.138            \\
\textbf{Reported Confidence}        & 7.31 $\pm$ 0.29 & 7.28 $\pm$ 0.29 & 0.862 & 0.046            \\
\textbf{\begin{tabular}[c]{@{}c@{}}Reported Trust in Model\end{tabular}} & 8.00 $\pm$ 0.45 & 6.87 $\pm$ 0.61 & \textbf{0.006} & 0.754\\ \hline
\end{tabular}%
}
\caption{D-RAPS vs RAPS: All Examples}
\label{tab:d_raps_vs_raps_all}
\end{table}


\vspace{-2mm}

\begin{table}[H]
\centering
\resizebox{8.5cm}{!}{
\begin{tabular}{ccccc}
\hline
\textbf{Metric}            & \textbf{D-RAPS}   & \textbf{RAPS}      & \textbf{$p$ value} & \textbf{Effect Size}\\ \hline
\textbf{Accuracy}          & 0.88 $\pm$ 0.05 & 0.81 $\pm$ 0.04  & 0.058 & 0.508           \\ 
\textbf{Reported Utility} & 7.93 $\pm$ 0.39 & 6.19 $\pm$ 0.62  & \textbf{$<$ 0.001} & 1.211            \\
\textbf{Reported Confidence}        & 7.78 $\pm$ 0.33 & 7.31 $\pm$ 0.34 & 0.059 & 0.507           \\ \hline
\end{tabular}%
}
\caption{D-RAPS vs RAPS: Non-Deferred Examples}
\label{tab:d_raps_vs_raps_ndef}
\end{table}

Tables $5$ and $6$ suggest that there is a statistically significant increase in expert accuracy when the D-CP scheme is employed, with borderline significance on non-deferred examples. Even though participants did not perform as well on deferred examples in general, we noticed that their accuracy was still higher than when they were shown CP sets, which contained misleading labels. Equally interestingly, on examples where both RAPS and D-RAPS sets contain the ground truth label (i.e. the non-deferred examples), the perceived utility of D-CP sets is higher ($p< 0.001$). As D-RAPS sets are smaller, this shows that, for the same confidence level, smaller set sizes are more useful to experts and therefore a preferred choice for human-AI teams. Table $5$ also shows a statistically significant difference in reported trust in the model between D-RAPS and RAPS. These are important considerations for real world human-AI teams. We provide further results in the Supplementary Material which show that participants display a bias towards towards incorrect predictions shown in larger CP sets, warranting caution when deploying models with large CP sets in human-AI teams.
\\
\noindent\textbf{Takeaway:} \textit{There are statistically significant improvements increases in reported utility, trust, and accuracy in the model when the D-CP scheme is employed.}
\section{Conclusion}
In this paper, we explored the importance of set valued predictions for human-AI teams. We first showed experts find CP predictive sets more useful than Top-1 predictions. However, CP set sizes can be very large for some examples, especially in large label spaces. Thus, we motivate the need for combining the ideas of learning to defer and set valued predictions. We introduce D-CP, a general practical scheme that defers some examples and performs CP on others. Empirical and theoretical evidence shows that the scheme provides smaller set sizes on non-deferred examples for any CP method compared to performing CP on all examples. The scheme allows the model and expert to have a mutually beneficial relationship by leveraging the expert and the model's respective strengths. Our human subject experiments show that, compared to CP, experts find the smaller D-CP predictive sets more useful, the model more trustworthy, and are more accurate. We hope that this informs a) future research on improved deferral policies that consider the predictive uncertainty of the model and b) larger scale human evaluations that uncover specific, desirable properties of a predictive set.

\clearpage
\section*{Acknowledgements}
The authors thank the anonymous reviewers for their thorough feedback. UB acknowledges support from DeepMind and the Leverhulme Trust via the Leverhulme Centre for the Future of Intelligence (CFI), and from the Mozilla Foundation. AW acknowledges support from a Turing AI Fellowship under grant EP/V025279/1, The Alan Turing Institute, and the Leverhulme Trust via CFI.
\bibliographystyle{named}
\bibliography{ijcai22}

\begin{thebibliography}{}

\bibitem[\protect\citeauthoryear{Angelopoulos and
  Bates}{2021}]{Angelopoulos_Intro_2021}
Anastasios~N Angelopoulos and Stephen Bates.
\newblock A {G}entle {I}ntroduction to {C}onformal {P}rediction and
  {D}istribution-{F}ree {U}ncertainty {Q}uantification.
\newblock {\em arXiv preprint arXiv:2107.07511}, 2021.

\bibitem[\protect\citeauthoryear{Angelopoulos \bgroup \em et al.\egroup
  }{2020}]{Angelopoulos2020}
Anastasios~Nikolas Angelopoulos, Stephen Bates, Michael Jordan, and Jitendra
  Malik.
\newblock Uncertainty {S}ets for {I}mage {C}lassifiers {U}sing {C}onformal
  {P}rediction.
\newblock In {\em International Conference on Learning Representations}, 2020.

\bibitem[\protect\citeauthoryear{Bansal \bgroup \em et al.\egroup
  }{2019}]{bansal2019updates}
Gagan Bansal, Besmira Nushi, Ece Kamar, Daniel~S Weld, Walter~S Lasecki, and
  Eric Horvitz.
\newblock Updates in {H}uman-{AI} {T}eams: {U}nderstanding and {A}ddressing the
  {P}erformance/{C}ompatibility {T}radeoff.
\newblock In {\em Proceedings of the AAAI Conference on Artificial
  Intelligence}, volume~33, pages 2429--2437, 2019.

\bibitem[\protect\citeauthoryear{Bartlett and Wegkamp}{2008}]{Bartlett2008}
Peter~L Bartlett and Marten~H Wegkamp.
\newblock Classification with a reject option using a hinge loss.
\newblock {\em Journal of Machine Learning Research}, 9(8), 2008.

\bibitem[\protect\citeauthoryear{Bellotti}{2021}]{Bellotti}
Anthony Bellotti.
\newblock Optimized conformal classification using gradient descent
  approximation.
\newblock {\em CoRR}, abs/2105.11255, 2021.

\bibitem[\protect\citeauthoryear{Bondi \bgroup \em et al.\egroup
  }{2022}]{bondi2021role}
Elizabeth Bondi, Raphael Koster, Hannah Sheahan, Martin Chadwick, Yoram
  Bachrach, Taylan Cemgil, Ulrich Paquet, and Krishnamurthy Dvijotham.
\newblock Role of {H}uman-{AI} {I}nteraction in {S}elective {P}rediction.
\newblock {\em Proceedings of the AAAI Conference on Artificial Intelligence},
  2022.

\bibitem[\protect\citeauthoryear{Cortes \bgroup \em et al.\egroup
  }{2016}]{Corinna2016}
Corinna Cortes, Giulia DeSalvo, and Mehryar Mohri.
\newblock Learning with {R}ejection.
\newblock In Ronald Ortner, Hans~Ulrich Simon, and Sandra Zilles, editors, {\em
  Algorithmic Learning Theory}, pages 67--82, Cham, 2016. Springer
  International Publishing.

\bibitem[\protect\citeauthoryear{Lei \bgroup \em et al.\egroup
  }{2016}]{Lei2016}
Jing Lei, Max G'Sell, Alessandro Rinaldo, Ryan~J. Tibshirani, and Larry
  Wasserman.
\newblock Distribution-{F}ree {P}redictive {I}nference {F}or {R}egression.
\newblock {\em Journal of the American Statistical Association},
  113:1094--1111, 4 2016.

\bibitem[\protect\citeauthoryear{Lundberg \bgroup \em et al.\egroup
  }{2018}]{Lundberg2018}
Scott~M. Lundberg, Bala Nair, Monica~S. Vavilala, Mayumi Horibe, Michael~J.
  Eisses, Trevor Adams, David~E. Liston, Daniel King-Wai Low, Shu-Fang Newman,
  Jerry Kim, and Su-In Lee.
\newblock Explainable {M}achine-{L}earning {P}redictions for the {P}revention
  of {H}ypoxaemia {D}uring {S}urgery.
\newblock {\em Nature Biomedical Engineering}, 2(10):749--760, October 2018.

\bibitem[\protect\citeauthoryear{Madras \bgroup \em et al.\egroup
  }{2018}]{madras2018predict}
David Madras, Toniann Pitassi, and Richard Zemel.
\newblock Predict {R}esponsibly: {I}mproving {F}airness and {A}ccuracy by
  {L}earning to {D}efer.
\newblock In {\em Proceedings of the 32nd International Conference on Neural
  Information Processing Systems}, pages 6150--6160, 2018.

\bibitem[\protect\citeauthoryear{Mozannar and
  Sontag}{2020}]{mozannar2020consistent}
Hussein Mozannar and David Sontag.
\newblock Consistent {E}stimators for {L}earning to {D}efer to an {E}xpert.
\newblock In {\em International Conference on Machine Learning}, pages
  7076--7087. PMLR, 2020.

\bibitem[\protect\citeauthoryear{Okati \bgroup \em et al.\egroup
  }{2021}]{okati2021differentiable}
Nastaran Okati, Abir De, and Manuel Gomez-Rodriguez.
\newblock Differentiable {L}earning {U}nder {T}riage.
\newblock In {\em Advances in Neural Information Processing Systems}, 2021.

\bibitem[\protect\citeauthoryear{Papadopoulos}{2008}]{Papadopoulos08}
Harris Papadopoulos.
\newblock Inductive {C}onformal {P}rediction: {T}heory and {A}pplication to
  {N}eural {N}etworks.
\newblock In Paula Fritzsche, editor, {\em Tools in Artificial Intelligence},
  chapter~18. IntechOpen, Rijeka, 2008.

\bibitem[\protect\citeauthoryear{Peterson \bgroup \em et al.\egroup
  }{2019}]{cifar10h}
Joshua Peterson, Ruairidh Battleday, Thomas Griffiths, and Olga Russakovsky.
\newblock Human {U}ncertainty {M}akes {C}lassification {M}ore {R}obust.
\newblock In {\em 2019 IEEE/CVF International Conference on Computer Vision
  (ICCV)}, pages 9616--9625, 2019.

\bibitem[\protect\citeauthoryear{Romano \bgroup \em et al.\egroup
  }{2020}]{Romano2020}
Yaniv Romano, Matteo Sesia, and Emmanuel Candes.
\newblock Classification with {V}alid and {A}daptive {C}overage.
\newblock In H.~Larochelle, M.~Ranzato, R.~Hadsell, M.~F. Balcan, and H.~Lin,
  editors, {\em Advances in Neural Information Processing Systems}, volume~33,
  pages 3581--3591. Curran Associates, Inc., 2020.

\bibitem[\protect\citeauthoryear{Sadinle \bgroup \em et al.\egroup
  }{2016}]{Sadinle2016}
Mauricio Sadinle, Jing Lei, and Larry Wasserman.
\newblock Least {A}mbiguous {S}et-{V}alued {C}lassifiers with {B}ounded {E}rror
  {L}evels.
\newblock {\em Journal of the American Statistical Association}, 114:223--234,
  9 2016.

\bibitem[\protect\citeauthoryear{Straitouri \bgroup \em et al.\egroup
  }{2022}]{okati_cp}
Eleni Straitouri, Lequn Wang, Nastaran Okati, and Manuel~Gomez Rodriguez.
\newblock Provably improving expert predictions with conformal prediction.
\newblock {\em CoRR}, abs/2201.12006, 2022.

\bibitem[\protect\citeauthoryear{Stutz \bgroup \em et al.\egroup
  }{2022}]{stutz2022learning}
David Stutz, Krishnamurthy~Dj Dvijotham, Ali~Taylan Cemgil, and Arnaud Doucet.
\newblock Learning optimal conformal classifiers.
\newblock In {\em ICLR}, 2022.

\bibitem[\protect\citeauthoryear{Tibshirani \bgroup \em et al.\egroup
  }{2019}]{Tibshirani2019}
Ryan~J. Tibshirani, Rina~Foygel Barber, Emmanuel~J. Candes, and Aaditya Ramdas.
\newblock Conformal {P}rediction {U}nder {C}ovariate {S}hift.
\newblock {\em Advances in Neural Information Processing Systems}, 32, 4 2019.

\bibitem[\protect\citeauthoryear{Vovk \bgroup \em et al.\egroup
  }{2005}]{vovk2005}
Vladimir Vovk, Alex Gammerman, and Glenn Shafer.
\newblock {\em Algorithmic Learning in a Random World}.
\newblock Springer, 01 2005.

\bibitem[\protect\citeauthoryear{Wilder \bgroup \em et al.\egroup
  }{2021}]{Wilder2020}
Bryan Wilder, Eric Horvitz, and Ece Kamar.
\newblock Learning to {C}omplement {H}umans.
\newblock In {\em Proceedings of the Twenty-Ninth International Joint
  Conference on Artificial Intelligence}, IJCAI'20, 2021.

\bibitem[\protect\citeauthoryear{Zagoruyko and
  Komodakis}{2016}]{zagoruyko2017wide}
Sergey Zagoruyko and Nikos Komodakis.
\newblock Wide {R}esidual {N}etworks.
\newblock In Edwin R.~Hancock Richard C.~Wilson and William A.~P. Smith,
  editors, {\em Proceedings of the British Machine Vision Conference (BMVC)},
  pages 87.1--87.12. BMVA Press, September 2016.

\end{thebibliography}

\clearpage
\appendix
\section{Proofs}
\begin{theorem}
\label{theorem:theorem1}
Consider a deferral policy $\pi(x): \mathcal{X} \rightarrow \{0,1\}$ and a classification model $m_\theta(x): \mathcal{X} \rightarrow \mathcal{Y}$ acting on a dataset $\mathcal{D} = \{(X_1,Y_1),...(X_n,Y_n)\}$. Define some conformity measure $\tau(x,y)$ such that if $p(\hat{y}|\hat{x}) \geq p(y|x)$ then $\tau(\hat{x},\hat{y}) \geq \tau(x,y)$ for any softmax probabilities $p(\hat{y}|\hat{x}),p(y|x)$, labels $\hat{y},y \in \mathcal{Y}$, and inputs $\hat{x},x \in \mathcal{X}$. If the expected loss on non-deferred examples is lower than the original loss, i.e. $\mathbb{E}_{(x,y)|\pi(x) = 0}[\mathcal{L}(y,m_{\theta}(x))] \leq \mathbb{E}_{(x,y)}[\mathcal{L}(y,m_{\theta}(x))]$, then the average conformal predictive set of non-deferred examples will contain fewer incorrect labels on average. 
\end{theorem}
\begin{proof}
Because the expected loss on non-deferred examples is lower, we know that:
\begin{equation}
    \mathbb{E}_{(x,y)|\pi(x) = 0}[p(Y=y|x)] \geq \mathbb{E}_{(x,y)}[p(Y=y|x)]
\end{equation}
From our definition of the conformity measure $\tau(x,y)$ above:
\begin{equation}
    \mathbb{E}_{(x,y)|\pi(x) = 0}[\tau(Y=y,x)] \geq \mathbb{E}_{(x,y)}[\tau(Y=y,x)]
\end{equation}
for any ground truth label $y$ associated with an example $x$. Therefore,
\begin{alignat*}{2}
    & \mathbb{E}_{(x,y)|\pi(x) = 0}[\sum_{\substack{i=1 \\ i \neq y}}^K p(Y = i|x)] \leq \mathbb{E}_{(x,y)}[\sum_{\substack{i=1 \\i \neq y}}^K p(Y = i|x)]\\
    &\Rightarrow \mathbb{E}_{(x,y)|\pi(x) = 0}[\sum_{\substack{i=1\\i\neq y}}^K \tau(Y=i,x)] \leq \mathbb{E}_{(x,y)}[\sum_{\substack{i=1\\i\neq y}}^K \tau(Y=i,x)]\\
\end{alignat*}
Because $\mathbb{E}_{(x,y)|\pi(x) = 0}[\tau(Y=y,x)] \geq \mathbb{E}_{(x,y)}[\tau(Y=y,x)]$, 
$\tau^\prime_{\alpha} = \textrm{Quantile}(\alpha, \{\tau(Y=y,x)|(x,y) \in \mathcal{D}, \pi(x) = 0\}) \geq \tau_{\alpha} = \textrm{Quantile}(\alpha,\{\tau(Y=y,x)|(x,y) \in \mathcal{D}\})$ for any user defined error tolerance $\alpha \in [0,1]$. Thus:
\begin{alignat*}{2}
    &\text{\scriptsize$\mathbb{E}_{(x,y)|\pi(x) = 0}\left[\sum_{\substack{i=1\\i\neq y}}^K\mathbb{I}_{ \tau(Y=i,x)\geq \tau^\prime_{\alpha}}\right] \leq \mathbb{E}_{(x,y)|\pi(x) = 0}\left[\sum_{\substack{i=1\\ i\neq y}}^K\mathbb{I}_{\tau(Y=i,x)\geq \tau_{\alpha}}\right]$} \\
    &\quad \quad \quad \quad \quad \quad \quad \quad \quad \quad \quad \quad \text{\scriptsize$\leq\mathbb{E}_{(x,y)}\left[\sum_{\substack{i=1\\ i\neq y}}^K\mathbb{I}_{\tau(Y=i,x)\geq \tau_{\alpha}}\right]$}
\end{alignat*}
This implies:
\begin{alignat*}{2}
    &\mathbb{E}_{(x,y)|\pi(x) = 0}[|\{\hat{y}|\tau(x,\hat{y}) \geq \tau^\prime_{\alpha}, \hat{y} \neq y\}|] \leq\\
    &\mathbb{E}_{(x,y)}[|\{\hat{y}|\tau(x,\hat{y}) \geq \tau_{\alpha}, \hat{y} \neq y\}|]
\end{alignat*}
\end{proof}
\vspace{1cm}
\section{Coverage Guarantees and Statistical Efficiency of D-CP}
For an Inductive Conformal Predictior (ICP), the coverage is defined as:
\begin{equation}
    \label{eqn:coverage}
    C = \frac{1}{n_{val}}\sum_{i=1}^{n_{val}}\mathcal{I}_{Y_i \in \Gamma(X_i,\tau_{cal})|\{(X_i,Y_i)\}_{i=1}^n}
\end{equation}
for a validation dataset of size $n_{val}$ and a calibration dataset of size $n$. While conformal prediction provides theoretical guarantees of the form in Equations $1$ and $2$, due to the finite number of samples and variations in test and training data distributions, ICP does not result in exact coverage in practice. \cite{Tibshirani2019} and \cite{Lei2016} report that the coverage of conformal intervals is highly concentrated around $1-\alpha$. Because D-CP ensures that samples in the calibration and validation sets remain exchangeable, we get similar coverage distributions for D-CP as we would for any CP method. This is illustrated in Figure \ref{fig:coverage}. 
\begin{figure}[H]
    \centering
    \includegraphics[scale=0.32]{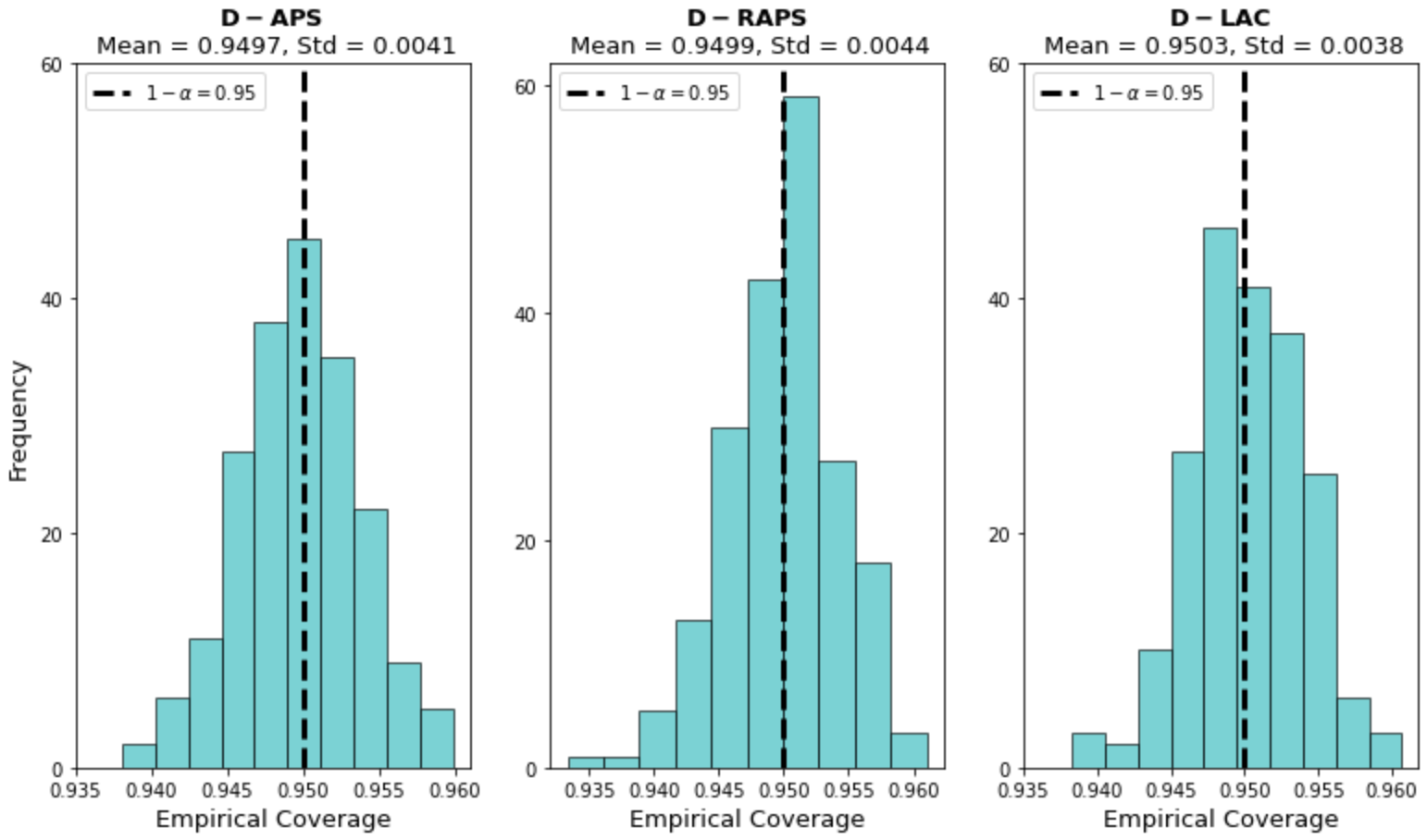}
    \includegraphics[scale=0.32]{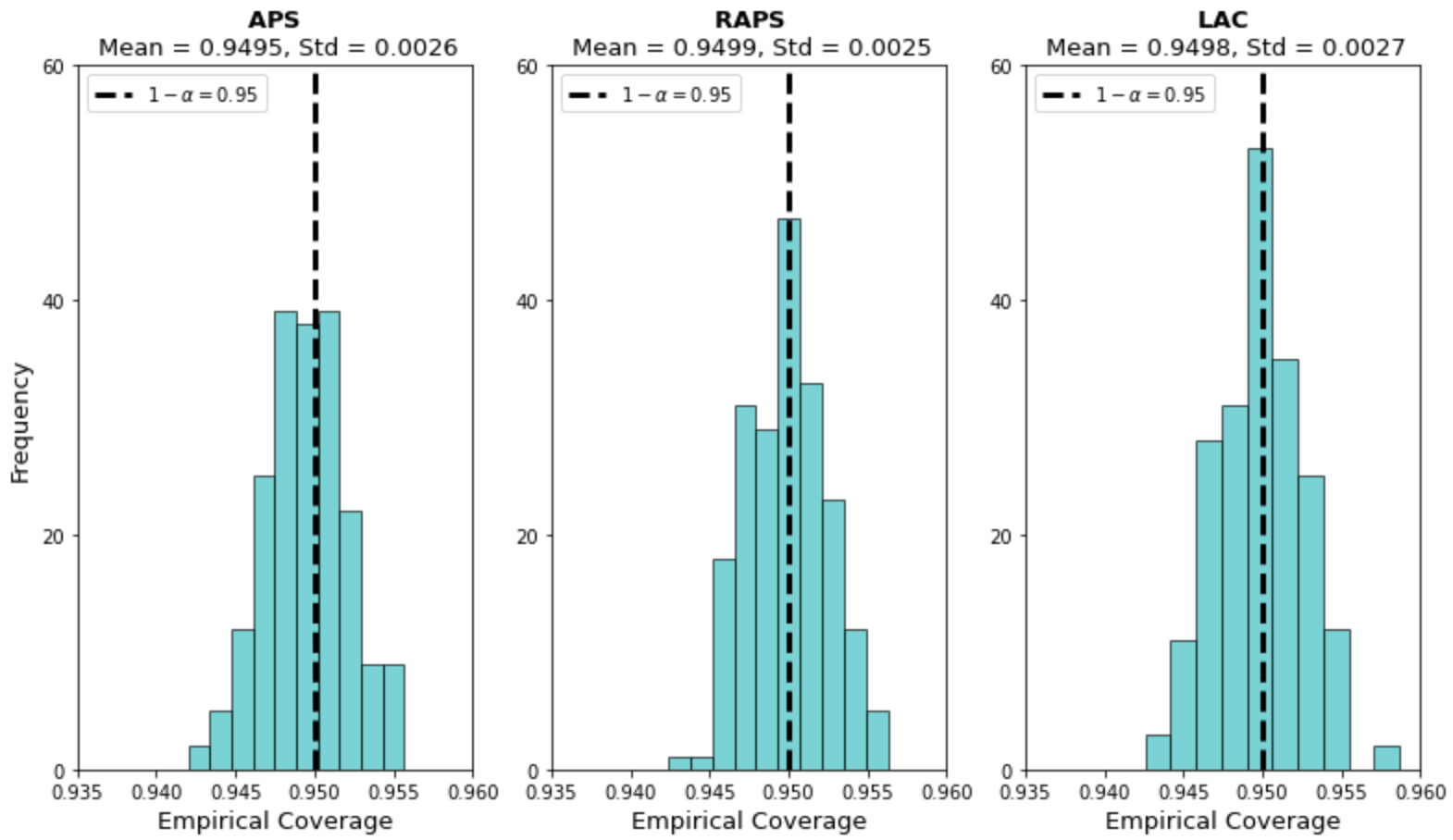}
    \caption{(top) Coverage Distribution on Non-Deferred Test Examples for Different D-CP Schemes: Deferral Rate $\beta \approx 0.2$. \\ (bottom) Coverage Distribution on All Test Examples for Different CP Schemes. For both D-CP and CP, $\alpha = 0.05$, Dataset = CIFAR-10H, Number of Trials = $200$, $n = 8000$, $n_{val} = 8000$}
    \label{fig:coverage}
\end{figure}
However, due to the reduced number of finite samples, we would expect a slight increase in the variance of the coverage of the estimator. This is evident in Figure \ref{fig:coverage}. \cite{Angelopoulos_Intro_2021} show that
the standard deviation of the obtained coverage in Equation \ref{eqn:coverage} can be expressed as:
\begin{align}
    \text{\scriptsize$\textrm{Std}(C) = \sqrt{\frac{(n+1-l)(n+n_{val}+1)l}{n_{val}(n+1)^2(n+2)}}
    = \mathcal{O}\bigg(\frac{1}{\sqrt{\min(n,n_{val})}}\bigg)$}
\end{align}
where $n$ and $n_{val}$ is the size of the calibration and test dataset respectively and $l = \lfloor (n+1)\alpha \rfloor$.
Given a deferral rate of $\beta$, the effective sizes of $n$ and $n_{cal}$ reduce by a factor of $1-\beta$ for D-CP, increasing the standard deviation of the average coverage by a factor of $\frac{1}{\sqrt{1-\beta}}$. The benefits of smaller predictive sets and human-AI complementarity therefore come at the price of a reduction of statistical efficiency. However, this is not a problem in practice as long as the model doesn't defer a large proportion of examples to an expert. \cite{Angelopoulos_Intro_2021} claim that a calibration size of $1000$ will be sufficient for most applications employing CP methods. For D-CP, given a model with, say a reasonable $20\%$ deferral rate, the calibration dataset need only be around $25\%$ larger than before to provide empirical coverage with the same variance as conventional CP methods.
\section{D-CP Algorithm for Experiments in Section 7}
Below, we present one instance of the D-CP algorithm that was used for experiments in this paper. Note that while the exact algorithm depends on the deferral policy being trained (e.g. using approaches in \cite{okati2021differentiable} or \cite{Wilder2020}), the main workflow followed is illustrated in Section $4$ in the paper.
\begin{algorithm}[H]
\caption{D-CP using the deferral policy in \protect\cite{mozannar2020consistent}}
\label{alg:dcp_monzannar}
\textbf{Input}: Classifier $m_\theta(x): \mathcal{X} \rightarrow \mathcal{Y}\cup\perp$, Training Set $\mathcal{D}$, Expert $h(x) \in \mathcal{Y}$ , Calibration Set $\mathcal{D}_{\textrm{cal}}$, Validation Set $\mathcal{D}_{\textrm{val}}$, Error Tolerance $\alpha$, Number of Epochs $N$, Learning Rate $\gamma$, Test Example $x_{\textrm{test}}$
\begin{algorithmic}[1]
\For {$i \in \{1,...N\}$}
\For {Batch $\mathcal{B} \in \mathcal{D}$}
    \State $\theta = \theta - \gamma\mathbb{E}_{(x,y) \in \mathcal{B}}[\nabla_{\theta} l(m_\theta(x),y,h(x))]$ \\ \Comment{Loss function in \cite{mozannar2020consistent}}
\EndFor
\EndFor
\State $D^\prime_{\textrm{cal}} = \varnothing$
\For {$(x,y) \in \mathcal{D}_{\textrm{cal}}$}
\If {$\textrm{argmax} m_\theta(x) \neq |\mathcal{Y}| + 1$}
\State $m^\prime_\theta = \textrm{softmax}(m_\theta)$ \quad \quad \quad\quad\Comment{Deferral Policy}
\State $m^\prime_\theta = \frac{m^\prime_\theta[1:|\mathcal{Y}|]}{1-m^\prime_\theta[|\mathcal{Y}|+1]}$ \quad\quad\quad\quad \Comment{Bayes' Rule}
\State $D^\prime_{\textrm{cal}} = D^\prime_{\textrm{cal}} \cup (x,y,m^\prime_\theta(x))$
\EndIf
\EndFor
\State $\tau_{cal}$ = $\alpha$ threshold conformity score learnt from conformal calibration on $\mathcal{D}^\prime_{cal}$
\If {$\textrm{argmax} m_\theta(x_{test}) \neq |\mathcal{Y}|+1$}
\State Output predictive set:\\ $\Gamma(x_{test}) = \{y|y\in \mathcal{Y}, \tau(x_{test},y) \geq \tau_{cal}\}$
\Else
\State Defer to expert $h(x_{test})$
\EndIf
\end{algorithmic}
\end{algorithm}
\section{Additional Results from Human Subject Experiments}
\subsection{Bias of Experts Towards Incorrect Model Predictions}
~\cite{bondi2021role} established for binary classifiers that model predictions influence expert decisions and that displaying incorrect predictions can cause experts to err in judgement when compared to purely deferring predictions. We report similar findings for set valued predictions in this paper.
To this end, we define the \textit{bias} towards incorrect predictions as the proportion of examples where an incorrect prediction made by an expert is found in the predictive set output by the model averaged across all subjects. That is, given experts $h$, examples $x$, the associated label $y(x)$, and the CP set $\Gamma(x)$:
\begin{equation}
    \text{Bias} = \mathbb{E}_{h,x}\big[\mathcal{I}_{h(x) \in \Gamma(x)}\mathcal{I}_{h(x) \neq y(x)}\big]
\end{equation}
\begin{table}[H]
\resizebox{8.6cm}{!}{%
\begin{tabular}{cccc}
\hline
\textbf{Metric} &
  \textbf{D-RAPS} &
  \textbf{\begin{tabular}[c]{@{}c@{}}RAPS \\ Non-Deferred Examples\end{tabular}} &
  \textbf{\begin{tabular}[c]{@{}c@{}}RAPS\\ Deferred Examples\end{tabular}} \\ \hline
Bias & 0.063 $\pm$ 0.035
   & 0.189 $\pm$ 0.046
   & 0.933 $\pm$ 0.069
   \\ \hline
\end{tabular}%
}
\caption{D-RAPS vs RAPS: Bias towards incorrect or misleading labels}
\label{tab:bias}
\end{table}
Comparing just the non-deferred examples (where both D-RAPS and RAPS sets contain the true label) we see that experts are much more biased towards incorrect predictions in RAPS sets than in D-RAPS sets. This is a consequence of RAPS sets containing more incorrect labels, which presents more scope for ambiguity. Another interesting observation is that on examples deferred by D-RAPS (whose RAPS sets contain only incorrect labels), expert are much more reliant on RAPS predictions. These findings warrant caution when deploying models with only CP wrappers in human-AI teams, as large, incoherent sets in critical settings can result in costly mistakes when expert bias their decisions heavily on model predictions. 
\subsection{Further Analysis of RAPS vs D-RAPS}
\begin{table}[H]
\vspace{-0.3cm}
\centering
\resizebox{8.75cm}{!}{%
\begin{tabular}{ccccccc}
\hline
\multirow{2}{*}{\textbf{Metric}} &
  \multirow{2}{*}{\textbf{RAPS}} &
  \multirow{2}{*}{\textbf{D-RAPS}} &
  \multirow{2}{*}{\textbf{N}} &
  \multirow{2}{*}{\textbf{p-value}} &
  \multirow{2}{*}{\textbf{Effect Size}} &
  \multirow{2}{*}{$\mathbf{N_{min}}$} \\
                              &        &        &    &        &       &         \\ \hline
\textbf{Accuracy (All)}       & 0.67 & 0.76 & 30 & \textbf{0.003} & 0.87 & 22  \\ \hline
\textbf{Accuracy (Easy)}      & 0.87 & 0.83 & 30 & 0.310 & 0.27 & 218 \\ \hline
\textbf{Accuracy (Difficult)} & 0.55 & 0.67 & 30 & \textbf{$<$ 0.001} & 1.04 & 16  \\ \hline
\end{tabular}%
}
\vspace{-0.1cm}
\caption{RAPS vs D-RAPS  Accuracy on All Examples. $N_{min}$ is the minimum sample size for each group needed for $p \leq 0.05$ with power $1-\beta = 0.8$ and $N$ is the experimental sample size of each group.}
\vspace{-0.3cm}
\label{tab:vrio}
\end{table}

Table 8 shows a power analysis on the results of the D-RAPS vs RAPS experiments. We divide the $15$ images chosen into $3$ difficulty groups - where difficulty is defined as the entropy of the model predictive distribution - and evaluate the statistical significance of the accuracy on the easiest and most difficult groups. It is seen that the accuracy increases the most on examples the model found difficult, which are by definition the most likely to be deferred. On the other hand, there is no increase in accuracy of easy examples. 
\newpage
\subsection{Human Subject Experiment Questions}
In order to aid reproducibility, we also show excerpts from the survey questionnaire for the RAPS vs Top-1 and D-RAPS vs RAPS experiments. 
\begin{itemize}
    \item Figures $9$ and $10$ are questions posed to participants across all experiments
    \item Figures $11$ and $12$ are questions posed to participants in the RAPS vs D-RAPS experiment.
    \item Figures $13$ and $14$ are questions posed to participants in the RAPS vs Top-1 experiment.
\end{itemize}

\begin{figure}[H]
    \centering
    \includegraphics[scale=0.4]{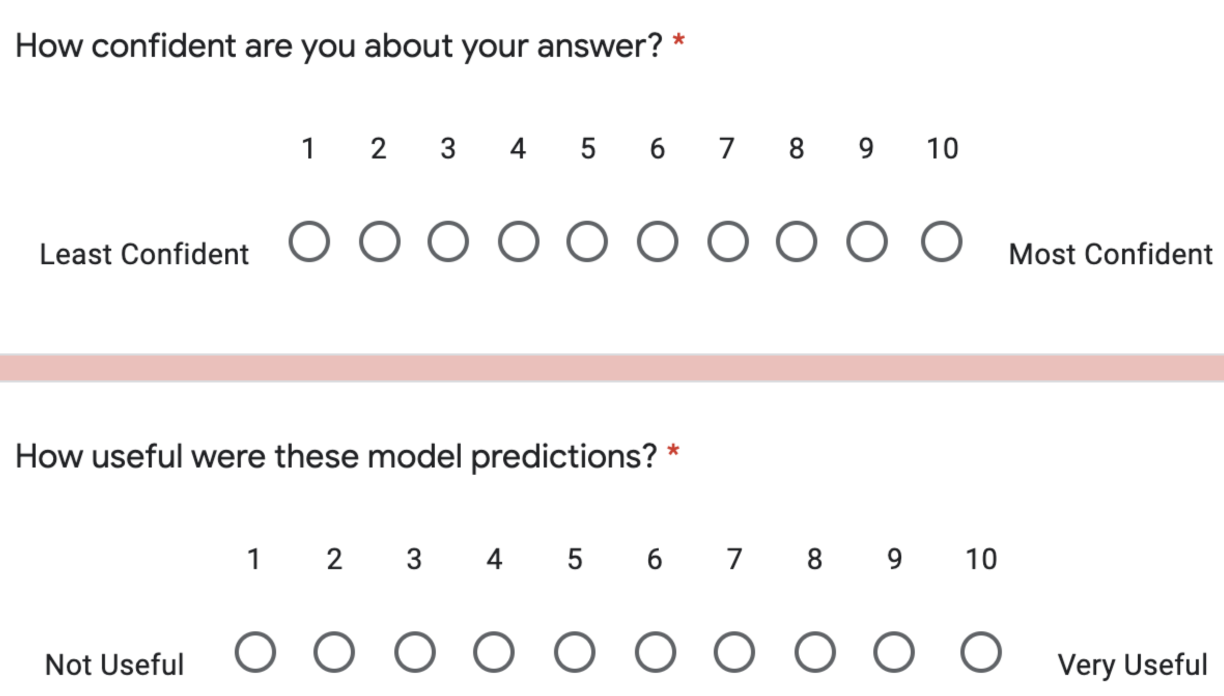}
    \caption{After each non-deferred example, we asked users these questions}
    \label{fig:survey_questions}
\end{figure}
\begin{figure}[H]
    \centering
    \includegraphics[scale=0.33]{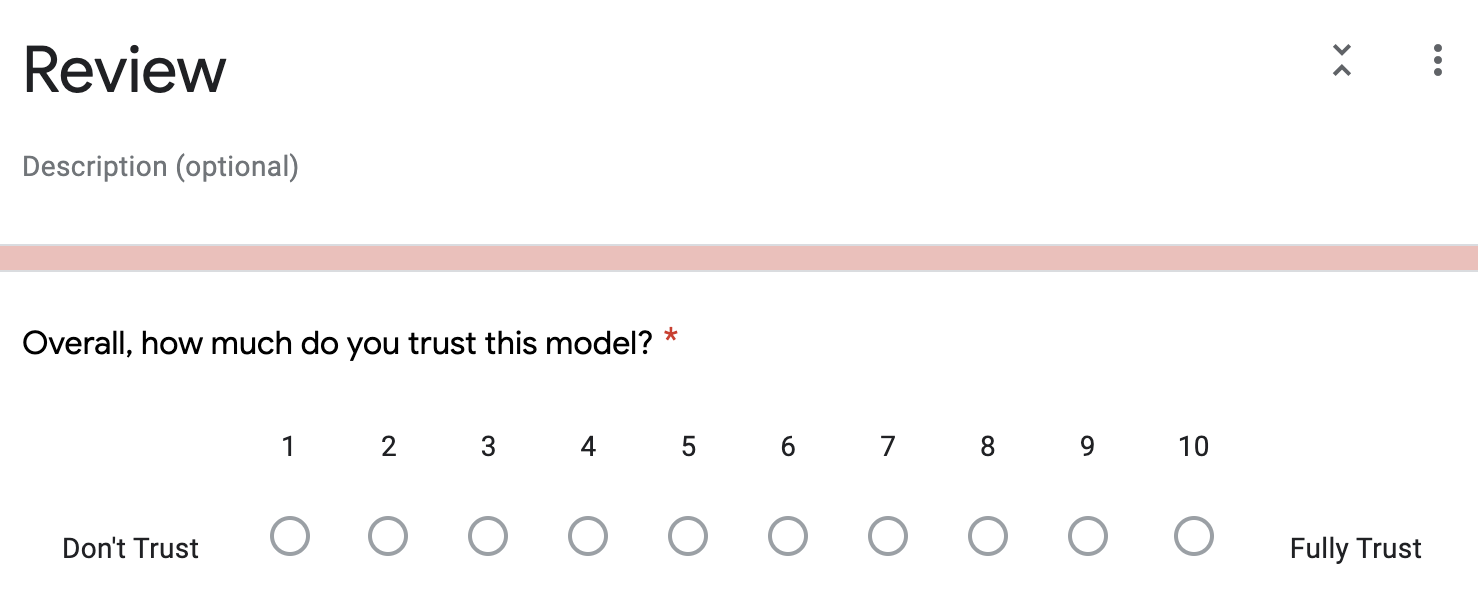}
    \caption{Additionally, at the end of the survey, we asked users to rate their trust in the model on a scale from $1-10$}
    \label{fig:trust}
\end{figure}

\begin{figure*}[htb]
    \centering
    \includegraphics[scale=0.4]{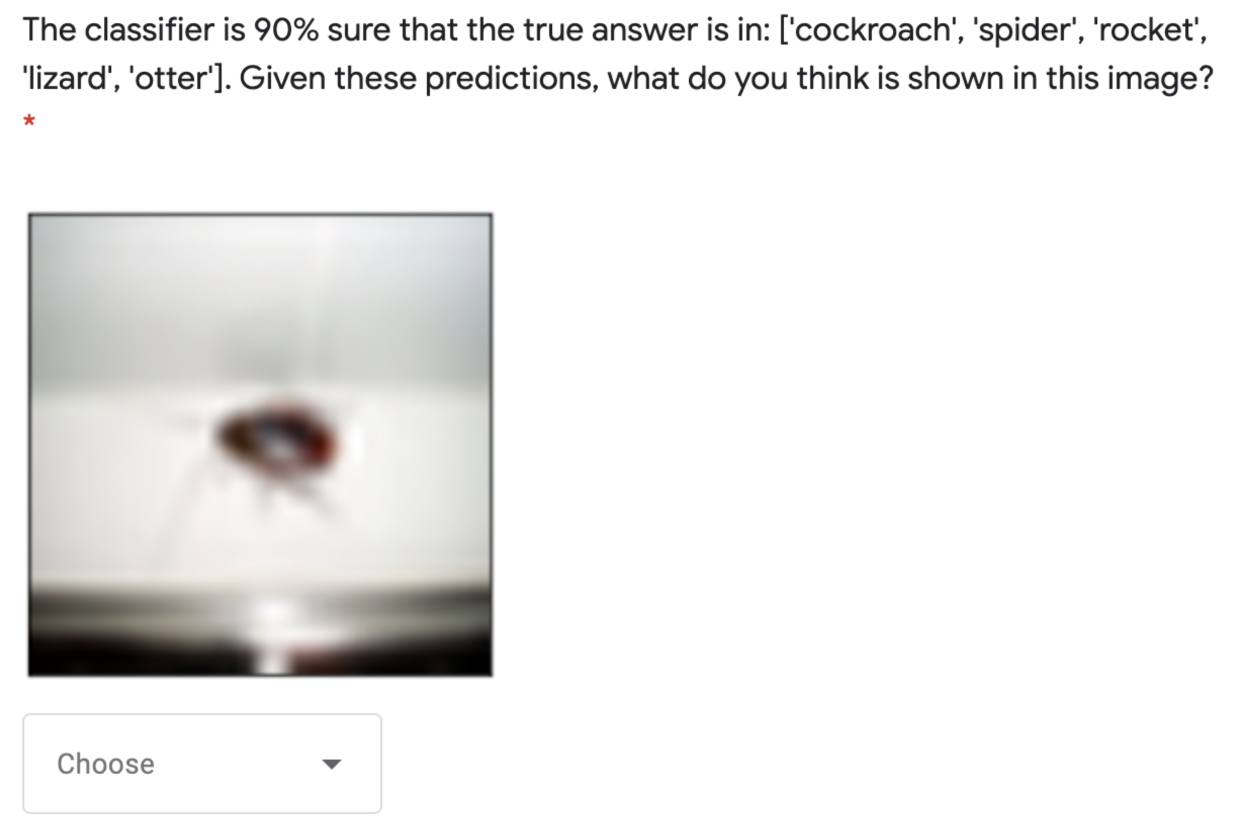}
     \includegraphics[scale=0.4]{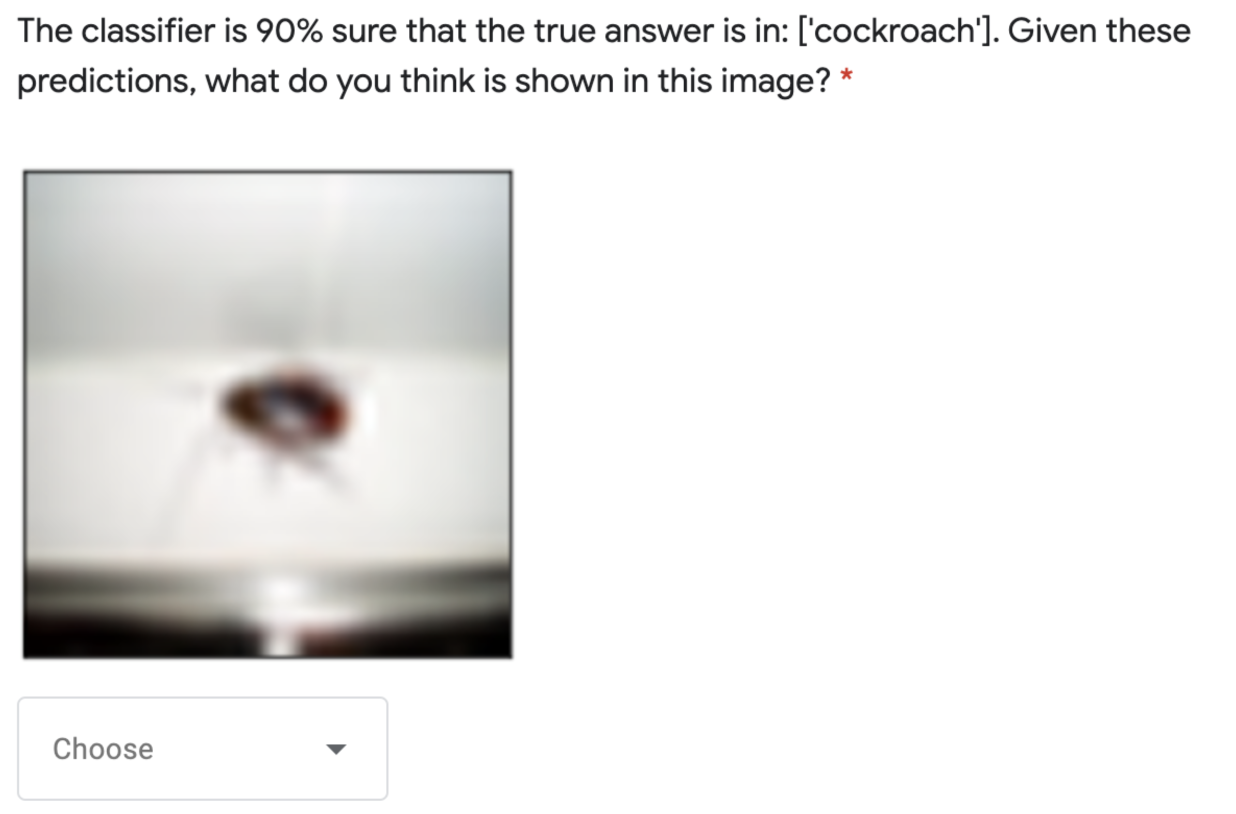}
    \caption{We first show an example of an image where both RAPS and D-RAPS sets contain the true label, but the D-RAPS set is smaller than the RAPS set. The true label for this image is: \textbf{Cockroach}. $15$ participants are shown the RAPS set (left image) and $15$ different participants are shown the D-RAPS set (right image)}
    \label{fig:RAPS_DRAPS_2_correct}
\end{figure*}
\begin{figure*}[htb]
    \centering
    \includegraphics[scale=0.4]{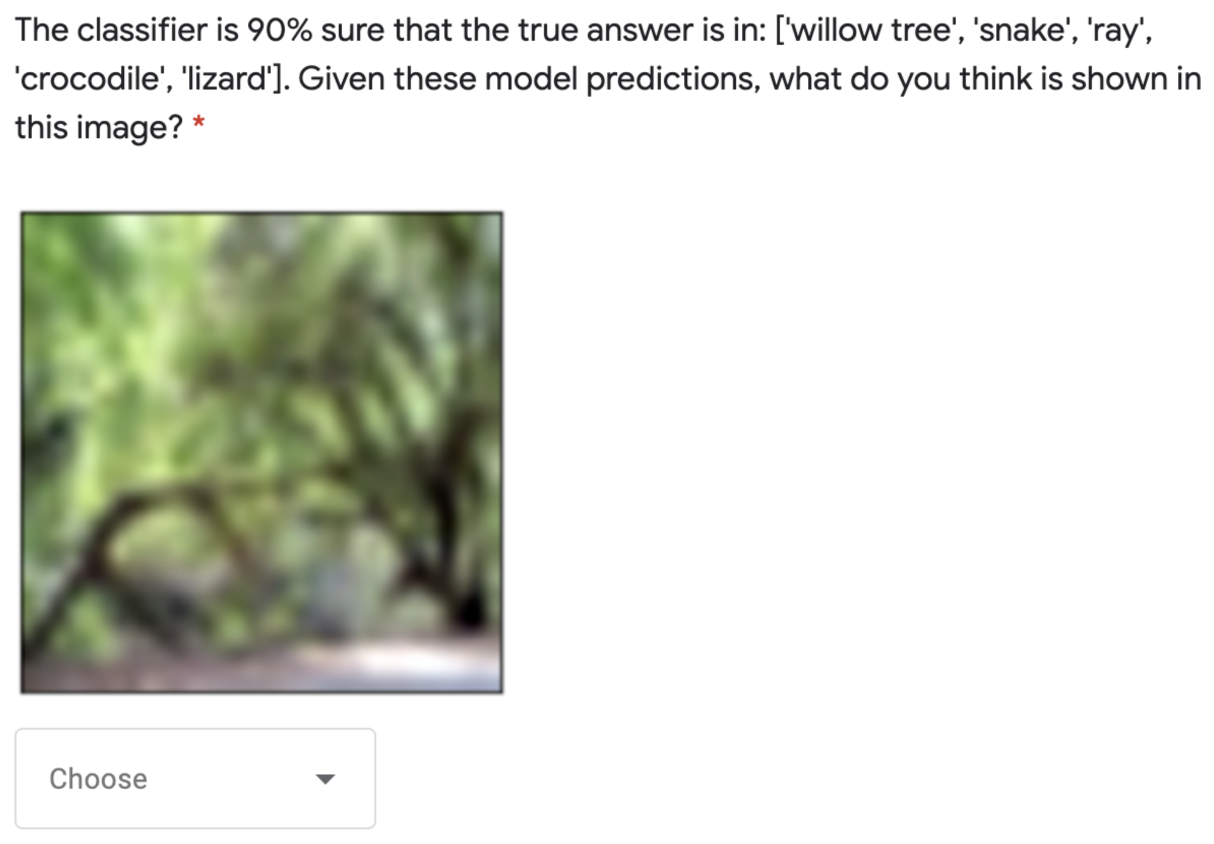}
     \includegraphics[scale=0.315]{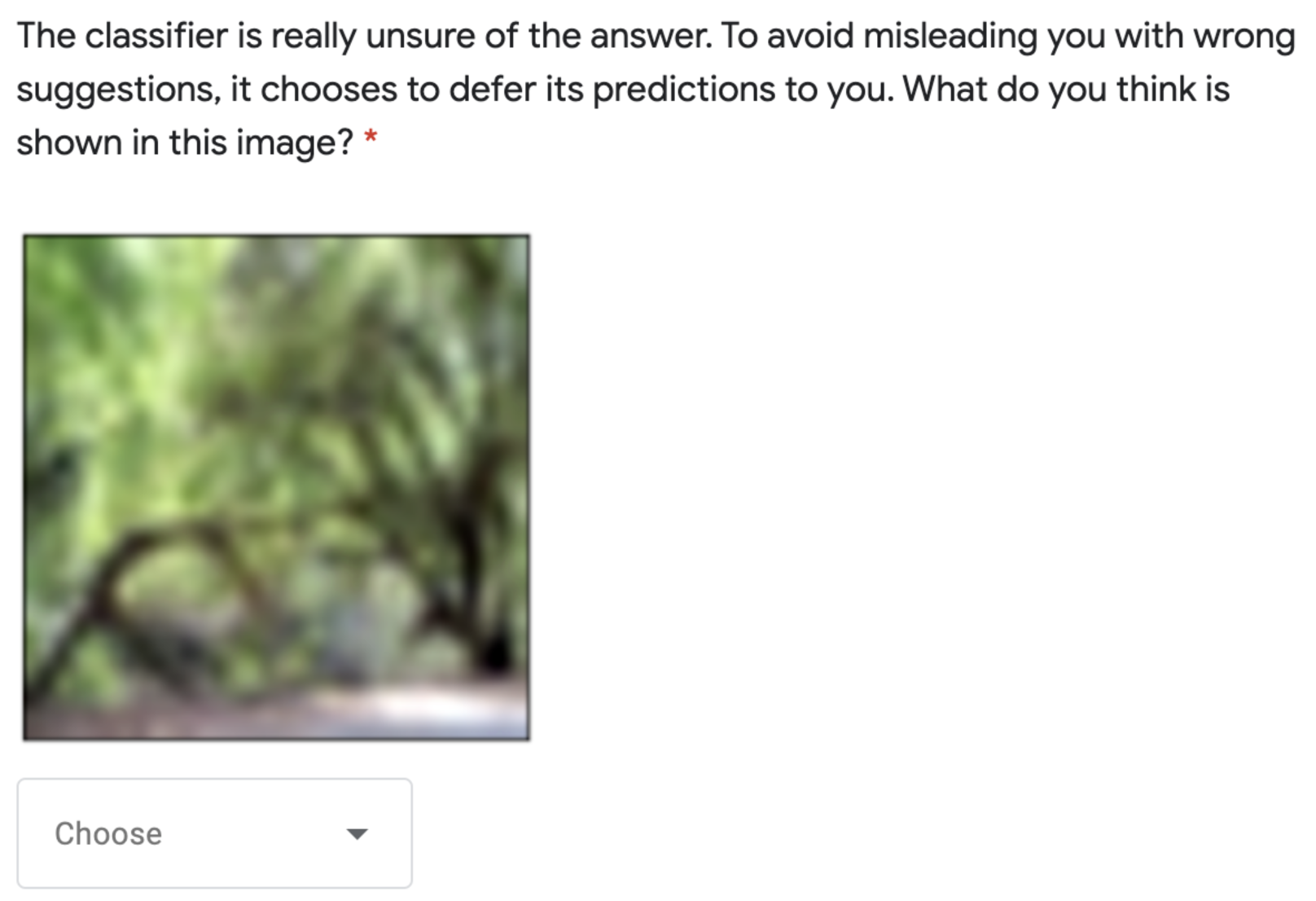}
    \caption{We also have images where the RAPS set provides incorrect and potentially misleading labels (such as 'willow tree') but where a D-RAPS set defers. The true label for this image is: \textbf{Forest}}
    \label{fig:RAPS_DRAPS_2_defer}
\end{figure*}
\begin{figure*}[htb]
    \centering
    \includegraphics[scale=0.4]{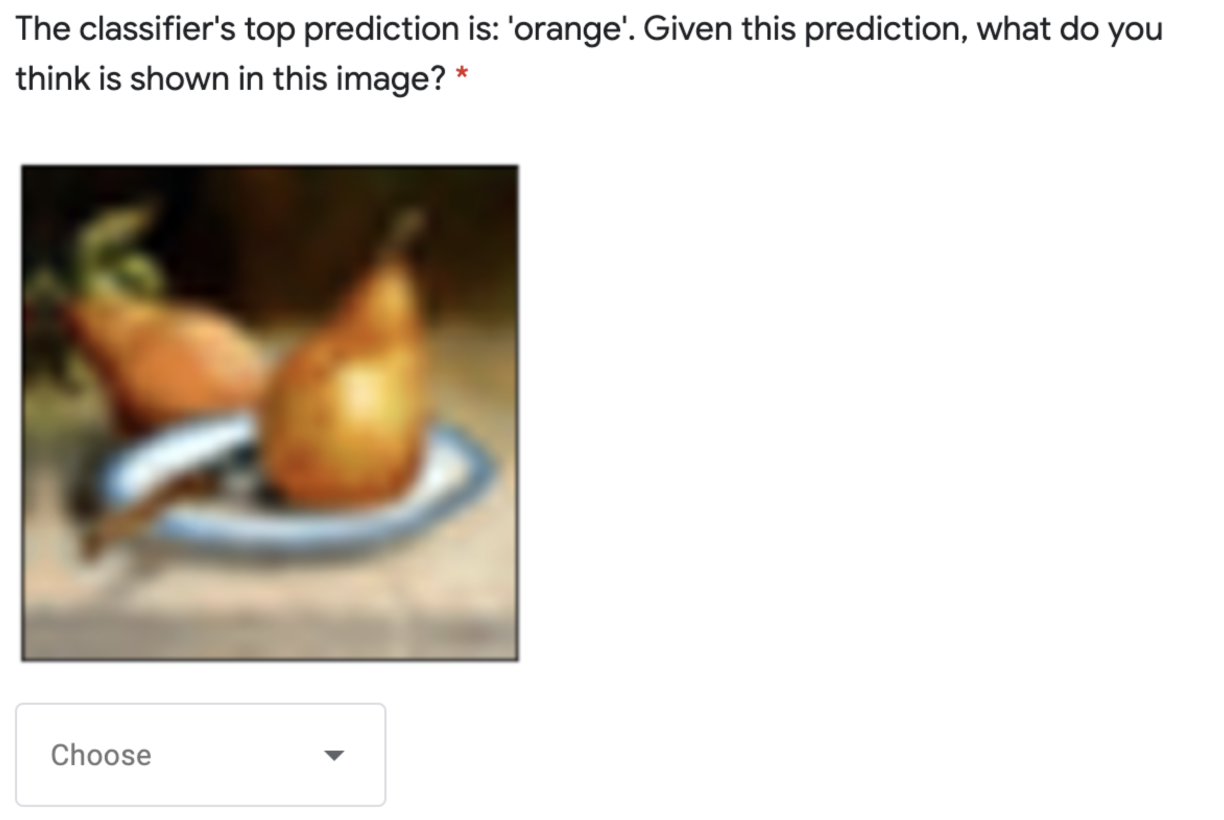}
     \includegraphics[scale=0.41]{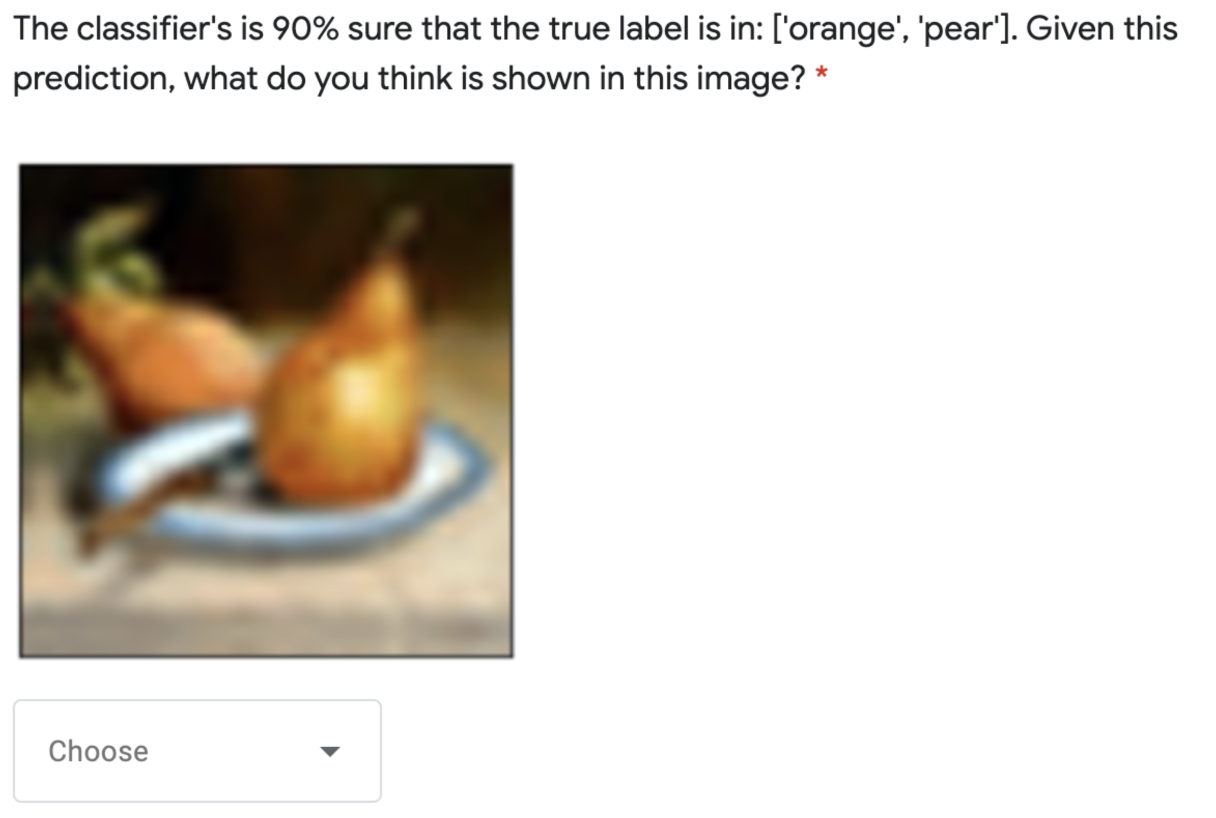}
    \caption{While evaluating Top-1 predictions (left) in comparison to RAPS (right), we choose some examples where Top-1 predictions are wrong but where the true label is contained in the RAPS set. We have $2$ such images for each difficulty quantile for a total of $6$ images. In this image, the true label is: \textbf{Pear}.}
    \label{fig:Top_1_RAPS_incorrect}
\end{figure*}
\begin{figure*}[htb]
    \centering
    \includegraphics[scale=0.41]{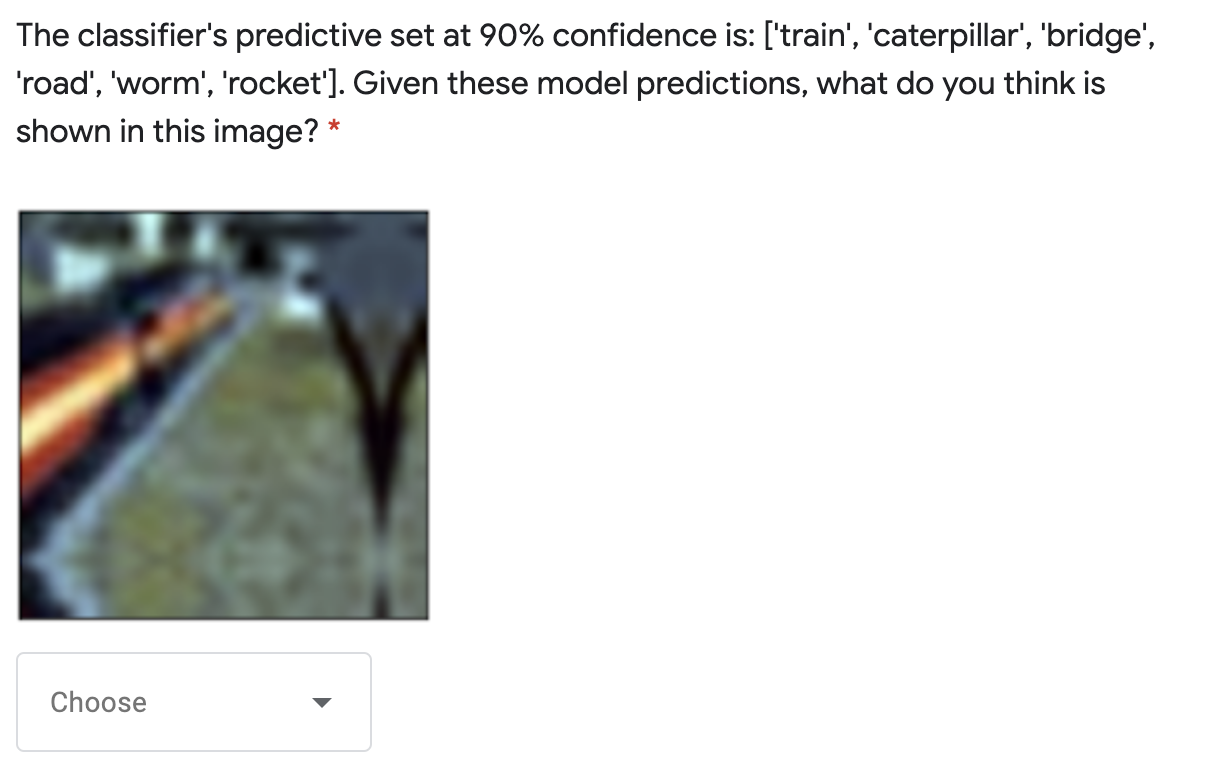}
     \includegraphics[scale=0.41]{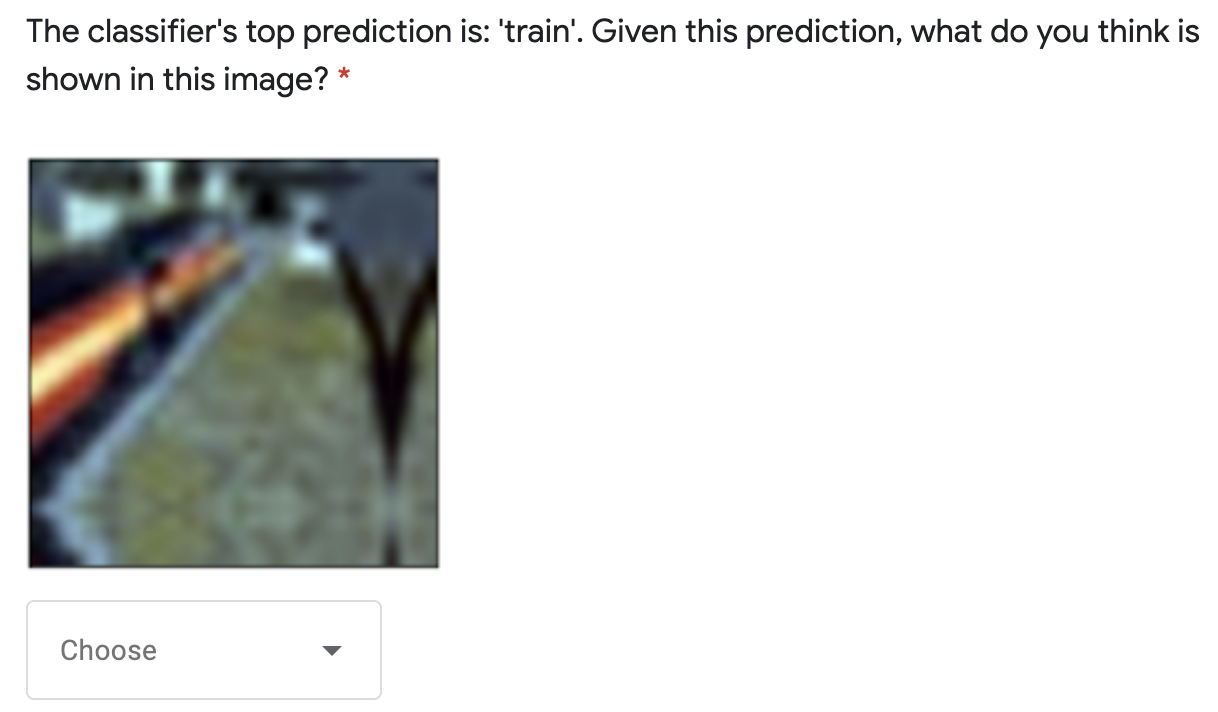}
    \caption{For the remaining $9$ images, the Top-1 predictions are correct and the RAPS sets contain the true label. Here, for example, the true label is: \textbf{Train}.}
    \label{fig:Top_1_RAPS_correct}
\end{figure*}

\end{document}